\documentclass[11pt,letterpaper]{amsart}
\usepackage{hyperref}
\usepackage{graphicx, color}
\usepackage{mathtools}
\usepackage{enumerate,enumitem}
\usepackage{bbm}
\usepackage[margin=1in]{geometry}
\usepackage{amssymb}
\usepackage[linesnumbered,ruled,vlined]{algorithm2e}
\usepackage{subcaption}
\usepackage{float}
\usepackage{tikz}
\usepackage{yhmath}
\usepackage{multirow}
\usepackage[numbers]{natbib}

\usepackage[flushleft]{threeparttable}
\usepackage{array,booktabs,makecell}
\usepackage{microtype}

\usepackage[capitalize,noabbrev]{cleveref}

\newtheorem{theorem}{Theorem}[section]
\newtheorem{proposition}[theorem]{Proposition}
\newtheorem{lemma}[theorem]{Lemma}

\theoremstyle{definition}
\newtheorem{definition}[theorem]{Definition}

\theoremstyle{remark}
\newtheorem{remark}[theorem]{Remark}

\DeclareMathOperator{\argmax}{argmax}
\DeclareMathOperator{\argmin}{argmin}
\DeclareMathOperator{\sign}{sign}
\DeclareMathOperator{\tr}{tr}

\newcommand{\R}{\mathbb{R}}
\newcommand{\E}{\mathbb{E}}

\newcommand{\prob}{\mathrm{P}}
\newcommand{\Var}{\mathrm{Var}}

\SetKwInput{KwInput}{Input}                
\SetKwInput{KwOutput}{Output}              
\SetKwInput{KwIni}{Initialize Parameters}

\begin{document}
\title[Sigma-Delta and Distributed Noise-Shaping Quantization Methods for RFFs]{Sigma-Delta and Distributed Noise-Shaping Quantization Methods for Random Fourier Features}

\author{Jinjie~Zhang}
\address{Department of Mathematics, University of California San Diego}
\email{jiz003@ucsd.edu}

\author{Harish~Kannan}
\address{Department of Mathematics, University of California San Diego}
\email{hkannan@ucsd.edu}

\author{Alexander~Cloninger}
\address{Department of Mathematics and Hal{\i}c{\i}o\u{g}lu Data Science Institute, University of California San Diego}
\email{acloninger@ucsd.edu}

\author{Rayan~Saab}
\address{Department of Mathematics and Hal{\i}c{\i}o\u{g}lu Data Science Institute, University of California San Diego}
\email{rsaab@ucsd.edu}
\maketitle

\begin{abstract}
We propose the use of low bit-depth Sigma-Delta and distributed noise-shaping methods for quantizing the Random Fourier features (RFFs) associated with shift-invariant kernels. We prove that our quantized RFFs -- even in the case of $1$-bit quantization --  allow a high accuracy approximation of the underlying kernels, and the approximation error decays at least polynomially fast as the dimension of the RFFs increases. We also show that the quantized RFFs can be further compressed, yielding an excellent trade-off between memory use and accuracy. Namely, the approximation error now decays exponentially as a function of the bits used. Moreover, we empirically show by testing the performance of our methods on several machine learning tasks that our method compares favorably to other state of the art quantization methods in this context.
\end{abstract}

\section{Introduction}\label{sec:intro}
Kernel methods have long been demonstrated as effective techniques in various machine learning applications, cf. \cite{shawe2004kernel,scholkopf2018learning}. Given a dataset $\mathcal{X}\subset\mathbb{R}^d$ with $|\mathcal{X}|=N$, kernel methods \emph{implicitly} map data points to a high, possibly infinite, dimensional feature space $\mathcal{H}$ by $\phi:\mathcal{X}\rightarrow\mathcal{H}$. However, instead of working directly on that space the inner products between feature embeddings can be preserved by a kernel function $k(x,y):= \langle \phi(x),\phi(y)\rangle_{\mathcal{H}}$ that coincides with the inner product. Nevertheless, in cases where $N$ is large, using nonlinear kernels for applications like, say, support vector machines (SVM) and logistic regression requires the expensive computation of the $N\times N$ Gram matrix of the data \cite{lin2007large}. 
In order to overcome this bottleneck, one popular approach is to ``linearize'' $k$ by using the random Fourier features (RFFs) originally proposed by \cite{Rahimi}, and in turn built on Bochner's theorem \cite{loomis2013introduction}. Given a continuous, shift-invariant real-valued kernel $k(x,y)=\kappa(x-y)$ with $\kappa(0)=1$, then $\kappa$ is the (inverse) Fourier transform of a probability measure $\Lambda$ over $\mathbb{R}^d$ and we have  
\begin{equation}\label{Bochner} 
\kappa(u)=\E_{\omega\sim\Lambda}\exp(i\omega^\top u)= \E_{\omega\sim\Lambda}\cos(\omega^\top u). 
\end{equation}

As an example, the radial basis function (RBF) kernel $k(x,y)=\exp(-\|x-y\|_2^2/2\sigma^2)$ corresponds to the multivariate normal distribution $\Lambda=\mathcal{N}(0,\sigma^{-2}I_d)$. Following \cite{Rahimi}, for a target dimension $m$, the associated RFFs (without normalization) are 
\begin{equation}\label{RFFs}
z(x):=\cos(\Omega^\top x+\xi)\in\mathbb{R}^m
\end{equation}
where $\Omega:=(\omega_1,\ldots,\omega_m)\in\mathbb{R}^{d\times m}$ is a random matrix generated as $\omega_j\overset{iid}{\sim} \Lambda$ and $\xi\in\mathbb{R}^m$ is a random vector with $\xi_j\overset{iid}{\sim} U([0,2\pi))$ for all $j$. Additionally, the identity $\E(\langle z(x),z(y)\rangle)=\frac{m}{2}k(x,y)$ implies that the inner product of low-dimensional features $\sqrt{\frac{2}{m}}z(x)$, $\sqrt{\frac{2}{m}}z(y)$ can approximate $k(x,y)$ in kernel-based algorithms. Learning a linear model on the (normalized) RFFs then amounts to using the approximation
\begin{equation}\label{approx-RFF}
    \widehat{k}_{\text{RFF}}(x,y) := \frac{2}{m}\langle z(x), z(y)\rangle
\end{equation}
as a reference kernel during training. For instance, performing linear SVM and linear ridge regression on RFFs winds up training nonlinear kernel-based SVM and ridge regression with $\widehat{k}_{\text{RFF}}$. It turns out that using RFFs in such a way with adjustable dimension $m$ can remarkably speed up training for large-scale data and alleviate the memory burden for storing the kernel matrix. As an additional and very important benefit, the entire kernel function $k$ is approximated accurately, i.e., the approximation error $|k(x,y)-\widehat{k}_{\text{RFF}}(x,y)|$ has been shown to be small, particularly when $m$ is large, e.g., in \cite{rudi2017generalization, bach2017equivalence, sutherland2015error, sriperumbudur2015optimal, avron2017faster, avron2017random}. 

The need for large $m$ for guaranteeing good generalization performance on large datasets \cite{tu2016large, JMLR:v20:17-026, agrawal2019data, liu2020random} provides an opportunity for further savings in memory usage. Rather than store the RFFs in full precision, quantization methods have been proposed to encode RFFs \eqref{RFFs} into a sequence of bits and subsequently approximate $k(x,y)$ by taking inner product between quantized RFFs, thereby introducing a new level of approximation. One of our goals is to propose quantization techniques that favorably trade off approximation accuracy against number of bits used. 

\subsection{Related Work} 
To make the discussion more precise, let us start by defining the $2K$-level quantization alphabet that we use throughout  as
\begin{equation}\label{alphabet}
\mathcal{A} = \Bigl\{\frac{a}{2K-1}\, \Big|\, a= \pm 1,\pm 3, \ldots, \pm (2K-1)  \Bigr\},
\end{equation}
and note that one can use $b:=\log_2(2K)$ bits to represent each element of $\mathcal{A}$. The goal of quantization in the RFF context is to map $z(x)=\cos(\Omega^T x+ \xi) \in \R^m \mapsto q(x) \in \mathcal{A}^m$. We will be interested in very small values of $K$, particularly $K=1$, which corresponds to very few bits per RFF sample.  

It is natural to start our discussion of quantization methods with the simplest quantizer, namely  memoryless scalar quantization (MSQ),  where we round each coordinate of the input vector $z\in\mathbb{R}^m$ to the nearest element in $\mathcal{A}$. Specifically, $Q_{\mathrm{MSQ}}:\mathbb{R}^m\rightarrow\mathcal{A}^m$ is defined by
\[
q_i :=(Q_{\mathrm{MSQ}}(z))_i:=\argmin_{v\in\mathcal{A}}|z_i-v|, \quad i=1,\ldots,m.
\]
Moreover, by setting $K=1$, one can get a binary embedding $Q_{\mathrm{MSQ}}(z)=\sign(z)$ with $\mathcal{A}=\{-1,1\}$ where $\sign$ is an element-wise operation. This yields the so-called one-bit universal quantizer  \cite{boufounos2013efficient, schellekens2020breaking} for RFFs, which generates a distorted (biased) kernel 
\begin{equation}\label{approx-universal}
\widehat{k}_q(x,y):=\frac{1}{m}\langle \sign(z(x)), \sign(z(y))\rangle.
\end{equation}

Although replacing the $\sign$ function in \eqref{approx-universal} by $Q_{MSQ}$ with $K>1$ and renormalizing the inner product correspondingly can alleviate the distortion, there are better choices in terms of approximation error. In \cite{li2021quantization}, a Lloyd-Max (LM) quantization scheme is designed based on the MSQ where, rather than use the evenly spaced alphabet in \eqref{alphabet}, one has to construct specific alphabets for different $K$. Recently with an eye towards asymmetric sensor network applications,  an asymmetric semi-quantized scheme (SemiQ)  was proposed in \cite{schellekens2020breaking}, and shown to be unbiased. It generates  $\widehat{k}_s(x,y)$, which is an inner product between an \emph{unquantized} RFF vector and a quantized one, i.e.
\begin{equation}\label{approx-semi}
    \widehat{k}_{s}(x,y) := \frac{\pi}{2m}\langle z(x), Q_{\mathrm{MSQ}}(z(y))\rangle.
\end{equation}
However, this asymmetric setting is restrictive on many kernel machines because it only works for the inference stage and the model still has to be trained based on unquantized RFFs. Another unbiased quantization scheme resorts to injecting randomness into the quantization, and is known as randomized rounding \cite{zhang2019low}, or stochastic quantization (StocQ)  \cite{li2021quantization}. Specifically, for each $z\in\mathbb{R}$, one chooses the two consecutive points $s, t\in\mathcal{A}$ with $z\in[s, t]$. Then one randomly assigns the quantization via  $\prob(Q_{\mathrm{StocQ}}(z)=s) = \frac{t-z}{t-s}$, $\prob(Q_{\mathrm{StocQ}}(z)=t) = \frac{z-s}{t-s}$. It follows that 
\begin{equation}\label{approx-StocQ}
    \widehat{k}_{\mathrm{StocQ}}(x,y) := \frac{2}{m}\langle Q_\mathrm{StocQ}(z(x)), Q_\mathrm{StocQ}(z(y))\rangle 
\end{equation}
where $Q_\mathrm{StocQ}$ operates on each component separately. 
Due to the Bernoulli sampling for $Q_\mathrm{StocQ}$, the quantization process involves additional randomness for each dimension of RFFs, which leads to extra variance especially in the case of binary embedding, i.e., $b=1$. Nevertheless, the kernel approximation error for $\widehat{k}_s$ and $\widehat{k}_{\text{StocQ}}$ is bounded by $O(m^{-1/2})$ with high probability, see \cite{schellekens2020breaking,zhang2019low}. 

\subsection{Methods and Contributions}
We explore the use of $\Sigma\Delta$ \cite{daubechies2003approximating, deift2011optimal,gunturk2003one} and distributed noise-shaping \cite{chou2016distributed, chou2017distributed} quantization methods on RFFs. These techniques,  explicitly defined and discussed in Section~\ref{sec:quantization} and Appendix~\ref{appendix:quantization}, yield superior performance to methods based on scalar quantization in contexts ranging from bandlimited function quantization \cite{daubechies2003approximating, gunturk2003one}, to quantization of linear measurements \cite{benedetto2006sigma, benedetto2006second}, of compressed sensing measurements \cite{gunturk2013sobolev}, of non-linear measurements \cite{huynh2016accurate}, and even for binary embeddings that preserve (Euclidean) distances \cite{huynh2020fast, zhang2021faster}. It is therefore  natural to wonder whether they can also yield superior performance in the RFF context. Let $Q_{\Sigma\Delta}^{(r)}$ be the $r$-th order $\Sigma\Delta$ quantizer and let $Q_\beta$ be the distributed noise shaping quantizer with $\beta\in(1,2)$, and let $\widetilde{V}_{\Sigma\Delta}$ and $\widetilde{V}_\beta$ be their associated sparse condensation matrices defined in Section~\ref{sec:quantization}.  Then our method approximates kernels via
\begin{equation}\label{approx-sigma-delta} 
\widehat{k}^{(r)}_{\Sigma\Delta}(x,y) := \langle \widetilde{V}_{\Sigma\Delta}Q_{\Sigma\Delta}^{(r)}(z(x)), \widetilde{V}_{\Sigma\Delta}Q_{\Sigma\Delta}^{(r)}(z(y)) \rangle
\end{equation}
and 
\begin{equation}\label{approx-beta}
\widehat{k}_\beta(x,y) := \langle \widetilde{V}_\beta Q_\beta (z(x)), \widetilde{V}_\beta Q_\beta(z(y)) \rangle.
\end{equation}
Specifically, given large-scale data $\mathcal{T}$ contained in a compact set $\mathcal{X}\subset\mathbb{R}^d$, we put forward Algorithm~\ref{algorithm1} to generate and store quantized RFFs such that one can subsequently use them for training and inference using linear models.

\begin{algorithm}[ht]
   \caption{Quantized kernel machines}
   \label{algorithm1}
\DontPrintSemicolon
  \KwInput{Shift-invariant kernel $k$, alphabet $\mathcal{A}$, and training data $\mathcal{T}=\{x_i\}_{i=1}^N \subset\mathcal{X}$} 
  Generate random matrix $\Omega\in\mathbb{R}^{d\times m}$ and random vector $\xi\in\mathbb{R}^m$ as in \eqref{RFFs} \\
  \For{$i= 1$ \KwTo $N$}
  {
  $z_i \leftarrow \cos(\Omega^\top x_i +\xi)\in\mathbb{R}^m$ 
      \hfill $\triangleright$ Compute RFFs \\
  $q_i \leftarrow Q(z_i)\in\mathcal{A}^m$ 
     \hfill$\triangleright$ $Q=Q_{\Sigma\Delta}^{(r)}$ or $Q_\beta$ as in \eqref{sigma-delta-eq} and \eqref{beta-eq}  \\
  $y_i \leftarrow \widetilde{V}q_i$  \hfill$\triangleright$ Further compression with $\widetilde{V}=\widetilde{V}_{\Sigma\Delta}$ or $\widetilde{V}_\beta$ as in \eqref{def:condensation-normalized}
  } 
Store $\{y_i\}_{i=1}^N$ and use it to train kernel machines with a linear kernel, i.e. inner product
\end{algorithm}

For illustration, Appendix~\ref{appendix:approx-comparison} presents a pointwise comparison of above kernel approximations on a synthetic toy dataset. A summary of our contributions follows. 
\begin{itemize}
    \item We give the first detailed analysis of $\Sigma\Delta$ and distributed noise-shaping schemes for quantizing RFFs. Specifically, Theorem~\ref{thm:main-result} provides a uniform upper bound for the errors $|\widehat{k}^{(r)}_{\Sigma\Delta}(x,y)-k(x,y)|$ and $|\widehat{k}_\beta(x,y)-k(x,y)|$ over compact (possibly infinite) sets. Our analysis shows that the quantization error decays fast as $m$ grows. Additionally, Theorem \ref{specBoundtheorem} provides spectral approximation guarantees for first order $\Sigma\Delta$ quantized RFF approximation of kernels.
    \item  Our methods allow a further reduction in the number of bits used. Indeed, to implement   \eqref{approx-sigma-delta} and \eqref{approx-beta} in practice, one would store and transmit the condensed bitstreams $\widetilde{V}_{\Sigma\Delta}Q_{\Sigma\Delta}^{(r)}(z(x))$ or $\widetilde{V}_\beta Q_\beta(z(x))$. For example, since the matrices $\widetilde{V}_{\Sigma\Delta}$ are sparse and essentially populated by bounded integers, each sample can be represented by fewer bits, as summarized in Table~\ref{table:comparison}. 
    \item  We illustrate the benefits of our proposed methods in several numerical experiments involving  kernel ridge regression (KRR), kernel SVM, and two-sample tests based on maximum mean discrepancy (MMD) (all in Section~\ref{sec:experiment}). Our experiments show that $Q_{\Sigma\Delta}^{(r)}$ and $Q_\beta$ are comparable with the semi-quantization scheme and outperforms the other fully-quantized method mentioned above, both when we fix the number of RFF features $m$, and when we fix the number of bits used to store each quantized RFF vector.
\end{itemize}

\section{Noise Shaping Quantization Preliminaries}\label{sec:quantization}
The methods we consider herein are  special cases of noise shaping quantization schemes (see, e.g., \cite{chou2015noise}). For a fixed alphabet $\mathcal{A}$ and each dimension $m$, such schemes are associated with an $m \times m$ lower triangular matrix $H$ with unit diagonal, and are given by a map $Q: \R^m \to \mathcal{A}^m$ with $y\mapsto q$ designed to satisfy $y-q = H u$. The schemes are called \emph{stable} if $\|u\|_\infty \leq C$ where $C$ is independent of $m$. Among these noise shaping schemes, we will be interested in stable $r^{\text{th}}$ order $\Sigma\Delta$ schemes $Q_{\Sigma\Delta}^{(r)}$ \cite{gunturk2003one, deift2011optimal}, and distributed noise shaping schemes $Q_\beta$ \cite{chou2016distributed, chou2017distributed}.  
For example, in the case of $\Sigma\Delta$ with $r=1$, the entries $q_i, i=1,...,m$ of the vector $q=Q_{\Sigma\Delta}^{(1)}(y)$ are assigned iteratively via  
\begin{equation}\label{sigma-delta-eq}
\begin{cases}
u_0=0,\\
q_i=Q_{\mathrm{MSQ}}\bigl(y_i+u_{i-1} \bigr),\\
u_i=u_{i-1}+y_i-q_i,
\end{cases}
\end{equation}
where $Q_{\mathrm{MSQ}}(z)=\argmin_{v\in\mathcal{A}}|z-v|$. This yields the difference equation $y-q=Du$ where $D$ is the first order difference matrix given by $D_{ij}=1$ if $i=j$, $D_{ij}=-1$ if $i=j+1$, and $0$ otherwise. Stable $\Sigma\Delta$ schemes with $r>1$, are more complicated to construct (see Appendix~\ref{appendix:quantization}), but satisfy 
\begin{equation}\label{diff-sigma-delta}
D^ru=y-q.
\end{equation}
On the other hand, a distributed   noise-shaping quantizer $Q_\beta:\mathbb{R}^m\rightarrow\mathcal{A}^m$ converts the input vector $y\in\mathbb{R}^m$ to $q=Q_\beta(y)\in\mathcal{A}^m$ such that 
\begin{equation}\label{diff-beta}
 Hu = y - q
\end{equation}
where, again, $\|u\|_\infty\leq C$. Here, denoting the  $p\times p$ identity matrix by $I_p$ and the Kronecker product by $\otimes$, $H$ is a block diagonal matrix defined as $H:=I_p\otimes H_\beta \in\mathbb{R}^{m\times m}$ where $H_\beta\in\mathbb{R}^{\lambda\times\lambda}$ is given by $({H_\beta})_{ij}=1$ if $i=j$, $({H_\beta})_{ij}=-\beta$ if $i=j+1$, and $0$ otherwise.
Defining $\widetilde{H}:=I_m-H$, one can implement the quantization step $q=Q_\beta(y)$ via the following iterations for $i=1,2,\ldots,m$ \cite{chou2016distributed, chou2017distributed}:
\begin{equation}\label{beta-eq}
\begin{cases}
u_0=0,\\
q_i=Q_{\mathrm{MSQ}}\bigl(y_i+\widetilde{H}_{i,i-1}u_{i-1}\bigr),\\
u_i=y_i+\widetilde{H}_{i,i-1}u_{i-1}-q_i,
\end{cases}
\end{equation}  
where $Q_{\mathrm{MSQ}}(z)=\argmin_{v\in\mathcal{A}}|z-v|$. The stability of \eqref{beta-eq} is discussed in Appendix~\ref{appendix:quantization}. It is worth mentioning that since $Q_{\Sigma\Delta}^{(r)}$ and $Q_\beta$ are sequential quantization methods, they can not be implemented entirely in parallel. On the other hand,  blocks of size $\lambda$ can still be run in parallel. 
Next, we adopt the definition of a condensation operator in \cite{chou2016distributed, huynh2020fast, zhang2021faster}. 
\begin{definition}[$\Sigma\Delta$ condensation operator]
Let $p$, $r$, $\lambda$ be fixed positive integers such that $\lambda=r\widetilde{\lambda}-r+1$ for some integer $\widetilde{\lambda}$. Let $m=\lambda p$ and $v$ be a row vector in $\mathbb{R}^\lambda$ whose entry $v_j$ is the $j$-th coefficient of the polynomial $(1+z+\ldots+z^{\widetilde{\lambda}-1})^r$. Define the condensation operator $V_{\Sigma\Delta}\in\mathbb{R}^{p\times m}$ as $V_{\Sigma\Delta}:=I_p\otimes v.$ 
\label{condensationOP}
\end{definition} 
For example, when $r=1$, $\lambda=\widetilde{\lambda}$ and the vector $v\in\R^{\lambda}$ is simply the vector of all ones while when $r=2$, $\lambda=2\widetilde{\lambda}-1$ and $v=(1,2,\ldots, \widetilde{\lambda}-1, \widetilde{\lambda}, \widetilde{\lambda}-1,\ldots, 2,1)\in\mathbb{R}^\lambda$. 
\begin{definition}[Distributed noise-shaping condensation operator]
Let $p,\lambda$ be positive integers and fix $\beta\in (1,2)$. Let $m=\lambda p$ and $v_\beta:=(\beta^{-1}, \beta^{-2}, \ldots, \beta^{-\lambda})\in\mathbb{R}^\lambda$ be a row
vector. Define the distributed noise-shaping condensation operator $V_\beta\in\mathbb{R}^{p\times m}$ as $V_\beta:=I_p\otimes v_\beta.$ 
\end{definition}
\noindent We will also need the normalized condensation operators given by 
\begin{equation}\label{def:condensation-normalized}
   \widetilde{V}_{\Sigma\Delta}:=\frac{\sqrt{2}}{\sqrt{p}\|v\|_2}V_{\Sigma\Delta}, \qquad \widetilde{V}_\beta := \frac{\sqrt{2}}{\sqrt{p}\|v_\beta\|_2}V_\beta. 
\end{equation}
If $\widetilde{V}$ is either of the two normalized matrices in \eqref{def:condensation-normalized},  Lemma~\ref{identities} (Appendix~\ref{appendix:proofs}) shows that 
\begin{equation}
\E(\langle \widetilde{V}z(x), \widetilde{V}z(y)\rangle) =k(x,y).
\label{eq:expectation}\end{equation}  

\section{Main Results and Space Complexity}\label{sec:main-result} 
Our approach to quantizing RFFs given by \eqref{approx-sigma-delta} and \eqref{approx-beta} is justified by \eqref{eq:expectation}, along with the observation that, for our noise-shaping schemes, we have $q=z-Hu$ with guarantees that $\|\widetilde{V}Hu\|_2$ is small.

Moreover, as we will see in Section \ref{sub:approx}, we are able to control the approximation error such that $\widehat{k}_{\Sigma\Delta}(x,y)\approx k(x,y)$ and $\widehat{k}_\beta(x,y)\approx k(x,y)$ hold with high probability. 
In fact 
Theorem \ref{thm:main-result} shows more: the approximation error of the quantized kernel
estimators in \eqref{approx-sigma-delta} and \eqref{approx-beta} have polynomial and exponential error decay respectively as a function of $m$, the dimension of the RFFs. Armed with this result, in Section \ref{sub:space} we also present a brief analysis of the space-complexity associated with our quantized RFFs, and show that the approximation error due to quantization decays exponentially as a function of the bits needed. 

Additionally, in various applications such as Kernel Ridge Regression (KRR), spectral error bounds on the kernel may be more pertinent than point-wise bounds. For example, it was shown in \cite{avron2017random, zhang2019low} that the expected loss of kernel ridge regression performed using an approximation of the true kernel is bounded by a function of the spectral error in the kernel approximation (Lemma 2 of \cite{avron2017random}, Proposition 1 of  \cite{zhang2019low}). In Theorem \ref{specBoundtheorem}, we provide spectral approximation guarantees for first order $\Sigma\Delta$ quantized RFF approximation of kernels, in the spirit of the analogous guarantees in \cite{zhang2019low} for stochastic quantization.

\subsection{Approximation error bounds}\label{sub:approx}
\subsubsection{Point-wise error bounds on the approximation}
We begin with Theorem \ref{thm:main-result}, with its proof in
Appendix~\ref{appendix:proofs}.
\begin{theorem}\label{thm:main-result}
Let $\mathcal{X}\subseteq\mathbb{R}^d$ be compact with diameter $\ell>0$ and $k:\mathcal{X}\times\mathcal{X}\rightarrow\mathbb{R}$ be a normalized, i.e. $k(0,0)=1$, shift-invariant kernel. Let $\Lambda$ be its corresponding probability measure as in \eqref{Bochner}, and suppose that the second moment $\sigma_\Lambda^2=\E_{\omega\sim \Lambda} \|\omega\|_2^2$ exists. Let $\beta\in(1,2)$, $p,r\in\mathbb{N}$, $\lambda=O(\sqrt{p\log^{-1} p})\in\mathbb{N}$, and $m=\lambda p$. For $x,y \in \mathcal{X}$, and {$b$-bit alphabet $\mathcal{A}$ in \eqref{alphabet} with $b=\log_2(2K)$}, consider the approximated kernels $\widehat{k}^{(r)}_{\Sigma\Delta}(x,y)$ and $\widehat{k}_\beta(x,y)$ defined as in \eqref{approx-sigma-delta} and \eqref{approx-beta} respectively. Then there exist positive constants $\{\alpha_i\}_{i=1}^{10}$ that are independent of $m,p,\lambda$ such that
\begin{equation}\label{uniform-bound-sigma-delta}
        \sup_{x,y\in\mathcal{X}}\bigl| \widehat{k}_{\Sigma\Delta}^{(r)}(x,y)-k(x,y)\bigr| \lesssim  \Bigl(\frac{\log p}{p}\Bigr)^{1/2} + \frac{\log^{1/2} p}{\lambda^{r-1}(2^b-1)} + \frac{1}{\lambda^{2r-1}(2^{b}-1)^2}
\end{equation}
holds with probability at least $1-\alpha_1 p^{-1-\alpha_2}-\alpha_3 \exp(-\alpha_4 p^{1/2}+\alpha_5 \log p)$, and
\begin{equation}\label{uniform-bound-beta}
        \sup_{x,y\in\mathcal{X}}\bigl| \widehat{k}_\beta(x,y)-k(x,y)\bigr| \lesssim  \Bigl(\frac{\log p}{p}\Bigr)^{1/2} + \frac{p^{1/2} }{\beta^{\lambda-1}(2^b-1)} + \frac{1}{\beta^{2\lambda-2}(2^{b}-1)^2}
\end{equation}
holds with probability exceeding $1-\alpha_6 p^{-1-\alpha_7}-\alpha_8 \exp(-\alpha_9 p^{1/2}+\alpha_{10} \log p)$.
\end{theorem}
Note that the first error term in \eqref{uniform-bound-sigma-delta}, \eqref{uniform-bound-beta} results from the condensation of RFFs, i.e. Theorem~\ref{uniform-bound-1}, while the remaining two error terms are due to the corresponding quantization schemes. 

\subsubsection{Spectral approximation guarantees for first order Sigma-Delta quantized RFFs}
We begin with a definition of a  $(\Delta_1, \Delta_2)$-spectral approximation of a matrix as the error bounds $\Delta_1$ and $\Delta_2$ play a key role in bounding the generalization error in various applications such as Kernel Ridge Regression (KRR) (Lemma 2 of \cite{avron2017random}, Proposition 1 of  \cite{zhang2019low}).
\begin{definition}
[$(\Delta_1,\Delta_2)$-spectral approximation] Given $\Delta_1, \Delta_2 >0$, a matrix $A$ is a $(\Delta_1,\Delta_2)$-spectral approximation of another matrix $B$ if $(1-\Delta_1)B \preccurlyeq A \preccurlyeq (1+\Delta_2)B$.
\end{definition}
For the tractability of obtaining spectral error bounds, in this section we consider a variation of the sigma-delta scheme for $r=1$. In particular, given a $b$-bit alphabet as in \eqref{alphabet} with $b=\log_2(2K)$, we consider the following first-order $\Sigma\Delta$ quantization scheme for a random Fourier feature vector $z(x) \in [-1,1]^m$ corresponding to a data point $x \in \mathbb{R}^d$, where, the state variable $(u_x)_0$ is initialized as a random number, i.e.
\begin{equation}
    \begin{split}
        (u_x)_0 & \sim U\left[-\frac{1}{2^b-1},\frac{1}{2^b-1}\right]\\
        q_{i+1} & = Q_{MSQ}((z(x))_{i+1} + (u_x)_{i})\\
        (u_x)_{i+1} & = (u_x)_{i} + (z(x))
        _{i+1} - q_{i+1}
    \end{split}
    \label{sigdelrand2}
\end{equation}
where $q \in \mathcal{A}^m$ represents the $\Sigma\Delta$ quantization of $z(x)$ and $(u_x)_0$ is drawn randomly from the uniform distribution on $\left[-\frac{1}{2^b-1},\frac{1}{2^b-1}\right]$.

Let $Q_{\Sigma\Delta}$ be the first order $\Sigma\Delta$ quantizer represented by \eqref{sigdelrand2} and let $\widetilde{V}_{\Sigma\Delta}$ be the associated sparse condensation matrix as in definition \ref{condensationOP}.  Then the elements of the corresponding approximation $\Hat{K}_{\Sigma\Delta}$ of the kernel $K$ is given by
\[
\widehat{K}_{\Sigma\Delta}(x,y) := \langle \widetilde{V}_{\Sigma\Delta}Q_{\Sigma\Delta}(z(x)), \widetilde{V}_{\Sigma\Delta}Q_{\Sigma\Delta}(z(y)) \rangle.
\]
Now, we state Theorem \ref{specBoundtheorem} whose proof can be found in Appendix \ref{specBoundtheoremProof}.
\begin{theorem}
\label{specBoundtheorem}
Let $\Hat{K}_{\Sigma\Delta}$ be an approximation of a true kernel matrix $K$ using $m$-feature first-order $\Sigma\Delta$ quantized RFF (as in \eqref{sigdelrand2}) with a $b$-bit alphabet (as in \eqref{alphabet}) and $m=\lambda p$. Then given $\Delta_1 \geq 0, \Delta_2 \geq \frac{\delta}{\eta}$ where $\eta>0$ represents the regularization and $\delta = \frac{8  +   \frac{26}{3p}}{\lambda(2^b-1)^2}$, we have
\[
\begin{split}
    & \mathbb{P}[ (1-\Delta_1)(K + \eta I) \preccurlyeq (\Hat{K}_{\Sigma\Delta} + \eta I) \preccurlyeq (1+\Delta_2)(K + \eta I)] \\
    & \geq \quad 1 - 4n\left[\exp(\frac{-p\eta^2\Delta_1^2}{4n\lambda(\frac{1}{\eta}(\|K\|_2 + \delta)+2\Delta_1/3)}) + \exp(\frac{-p\eta^2(\Delta_2 - \frac{\delta}{\eta})^2}{4n\lambda(\frac{1}{\eta}(\|K\|_2 + \delta)+2(\Delta_2 - \frac{\delta}{\eta})/3)}) \right]. 
\end{split}
\]
\end{theorem}
The above result differs from the spectral bound results presented in \cite{zhang2019low} for stochastic quantization in a particular aspect of the  the lower bound requirement on $\Delta_2$, namely, the lower bound for $\Delta_2$ in Theorem \ref{specBoundtheorem} for first order $\Sigma\Delta$ quantization has another controllable parameter $\lambda$ in addition to the number of bits $b$. Specifically, provided $8  >> \frac{26}{3p}$, we have $\delta \approx \frac{8}{\lambda(2^b-1)^2}$, which is monotonically decreasing in $\lambda$. 

\subsection{Space complexity} \label{sub:space} At first glance, Theorem \ref{thm:main-result} shows that $Q_\beta$ has faster quantization error decay as a function of $\lambda$ (hence $m$) as compared to $Q_{\Sigma\Delta}^{(r)}$. However, a further compression of the bit-stream resulting from the latter is possible, and results in a similar performance of the two methods from the perspective of bit-rate versus approximation error, as we will now show. 

Indeed, our methods entail training and testing linear models on condensed bitstreams 
$\widetilde{V}q\in \widetilde{V}\mathcal{A}^m\subset \R^p$ where $q$ is the quantized RFFs generated by $Q_{\Sigma\Delta}^{(r)}$ or $Q_\beta$, and $\widetilde{V}$ is the corresponding normalized condensation operator. 
Thus, when considering the space complexity associated with our methods, the relevant factor is the number of bits needed to encode $\widetilde{V}q$.
To that end, by storing the normalization factors in $\widetilde{V}$ (see \eqref{def:condensation-normalized}) separately using a constant number of bits, we can simply ignore them when considering space complexity. Let us now consider $b$-bit alphabets $\mathcal{A}$ with $b=\log_2(2K)$. Since the entries of $v$ are integer valued and $\|v\|_1=O(\lambda^r)$, one can store $\widetilde{V}_{\Sigma\Delta}q$ using $B:=O(p\log_2(2K\|v\|_1))=O(p(b+r\log_2\lambda))$ bits. {Then $\lambda^{-r} \approx 2^{-c B/p}$ and thus the dominant error terms in \eqref{uniform-bound-sigma-delta} decay exponentially as a function of bit-rate $B$.} On the other hand, for distributed noise shaping each coordinate of $\widetilde{V}_\beta q$ is a linear combination of $\lambda$ components in $q$, so $\widetilde{V}_\beta q$ takes on at most $(2K)^\lambda$ values. This implies we need $p\log_2(2K)^\lambda=mb$ bits to store $\widetilde{V}_\beta q$ in the worst case. 

\begin{remark}
Despite this tight upper bound for arbitrary $\beta\in(1,2)$, an interesting observation is that the number of bits used to store $\widetilde{V}_\beta q$ can be smaller than $mb$ with special choices of $\beta$, e.g., when  $\beta^k=\beta+1$ with integer $k>1$. For example, if $k=2$ and $b=1$, then $\beta=(\sqrt{5}+1)/2$ is the \emph{golden ratio} and one can see that $v_\beta=(\beta^{-1},\ldots,\beta^{-\lambda})$ satisfies $v_\beta(i)=v_\beta(i+1)+v_\beta(i+2)$ for $1\leq i\leq \lambda-2$. Since $b=1$, we have $q\in\{\pm 1\}^m$ and $\widetilde{V}_\beta q$ (ignoring the normalizer) can be represented by $p\log_2(\beta^\lambda)=m\log_2(\beta)<m$ bits. Defining the number of bits used to encode each RFF vector by $R:=m\log_2(\beta)$, then \eqref{uniform-bound-beta} shows that  $\beta^{-\lambda}=2^{-\lambda R/m}=2^{-R/p}$ dominates the error. In other words, up to constants, the error is essentially equal to the error obtained by a $\lambda$ bit MSQ quantization of a $p$-dimensional RFF embedding. 
\end{remark}

If we assume that each full-precision RFF is represented by $32$ bits, then the
storage cost per sample for both full-precision RFF and semi-quantized scheme $Q_{\text{SemiQ}}$ in \eqref{approx-semi} is $32m$. Because $Q_{\text{StocQ}}$ in \eqref{approx-StocQ} does not admit further compression, it needs $mb$ bits. A comparison of space complexity of different methods is summarized in Table~\ref{table:comparison}.

\begin{table}[ht] 
\caption{The memory usage to store each encoded sample.}
\label{table:comparison} 
\centering
\begin{threeparttable}\begin{tabular}{cccccc}
    \toprule
Method & RFFs & $Q_{\text{SemiQ}}$ & $Q_{\text{StocQ}}$ & $Q_{\Sigma\Delta}^{(r)}$ & $Q_\beta$ \\
     \midrule
Memory & $32m$ & $32m$ &  $mb$ & $O(p(b+r\log_2\lambda))$  & $mb{}^*$ \\
    \bottomrule
\end{tabular}
    \begin{tablenotes}
      \small
      \item  ${}^\star$ This can be reduced to $mb\log_2\beta$ for certain $\beta$. 
    \end{tablenotes}
  \end{threeparttable}
\end{table} 

\section{Numerical Experiments}\label{sec:experiment}
We have established that both $Q_{\Sigma\Delta}^{(r)}$ and $Q_\beta$ are memory efficient and approximate their intended kernels well. In this section, we will verify via numerical experiments that they perform favorably compared to other baselines on machine learning tasks.
\subsection{Kernel Ridge Regression}
Kernel ridge regression (KRR) \cite{murphy2012machine} corresponds to the ridge regression (linear least squares with $\ell_2$ regularization) in a reproducing kernel Hilbert space (RKHS). We synthesize $N=5000$ highly nonlinear data samples $(x_i,y_i)\in\mathbb{R}^5\times \mathbb{R}$ such that for each $i$, we draw each component of $x_i\in\mathbb{R}^5$ uniformly from $[-1,1)$ and use it to generate 
\[
y_i = f(x_i):=\gamma_1^\top x_i +\gamma_2^\top \cos(x_i^2) + \gamma_3^\top \cos(|x_i|) + \epsilon_i \]
where $\gamma_1=\gamma_2=\gamma_3=[1,1,\ldots, 1]^\top\in\mathbb{R}^5$, and $\epsilon_i\sim\mathcal{N}(0,\frac{1}{4})$. This is split into $4000$ samples used for training and $1000$ samples for testing. Given a RBF kernel $k(x,y)=\exp(-\gamma\|x-y\|_2^2)$ with $\gamma=1/d=0.2$, by the representer theorem, our predictor is of the form $\widehat{f}(x)=\sum_{i=1}^N\alpha_i k(x_i,x)$ where the coefficient vector $\alpha:=(\alpha_1,\ldots,\alpha_N)\in\mathbb{R}^N$ is obtained by solving $(K+\eta I_N)\alpha=y$. Here, $K=(k(x_i,x_j))\in\mathbb{R}^{N\times N}$ is the kernel matrix and $\eta =1$ is the regularization parameter. 

Since the dimension of RFFs satisfies $m=\lambda p$, there is a trade-off  between $p$ and $\lambda$. According to Theorem~\ref{thm:main-result}, increasing the embedding dimension $p$ can reduce the error caused by compressing RFFs, while larger $\lambda$ leads to smaller quantization error and makes the memory usage of $Q_{\Sigma\Delta}^{(r)}$ more efficient (see Table~\ref{table:comparison}). 
Beyond this, all hyperparameters, e.g. $\lambda$, $\beta$, are tuned based on  cross validation. 
In our experiment, we consider the kernel approximations $\widehat{k}_{\text{RFF}}$, $\widehat{k}_{\text{StocQ}}$, $\widehat{k}_{\Sigma\Delta}^{(1)}$ with $\lambda=15$, $\widehat{k}_{\Sigma\Delta}^{(2)}$ with $\lambda=15$, and $\widehat{k}_\beta$ with $\beta=1.9$, $\lambda=12$.  These are applied for both training (solving for $\alpha$) and testing (computing $\widehat{f}(x)$ based on $\alpha$), while the semi-quantized scheme $\widehat{k}_s$ is only used for testing and its coefficient vector $\alpha$ is learned by using $\widehat{k}_{\text{RFF}}$ on the training set. Furthermore, according to \cite{schellekens2020breaking}, $\widehat{k}_s$ can be used in two scenarios during the testing stage:
\begin{enumerate}
    \item Training data is unquantized RFFs while test data is quantized, i.e., $\widehat{f}(x)=\sum_{i=1}^N\alpha_i \widehat{k}_s(x_i,x)$;
    \item Quantize training data and leave testing points as RFFs, i.e., $\widehat{f}(x)=\sum_{i=1}^N\alpha_i \widehat{k}_s(x,x_i)$.
\end{enumerate} 
We summarize the KRR results averaging over $30$ runs for $b=1$ bit quantizers in Figure~\ref{fig:rkk_b1}, in which solid curves represent our methods and the dashed lines depict other baselines. Note that in both cases, the noise-shaping quantizer $Q_\beta$ achieves the lowest test mean squared error (MSE) among all quantization schemes, and it even outperforms the semi-quantization scheme $\widehat{k}_s$ with respect to the number of measurements $m$. Moreover, due to the further compression advantage, $Q_{\Sigma\Delta}^{(r)}$ and $Q_\beta$ are more memory efficient than the fully-quantized scheme $Q_{\text{StocQ}}$ in terms of the usage of bits per sample. More experiments for $b=2,3$ can be found in Appendix~\ref{appendix:extra-experiments}.

\begin{figure}[ht]
\begin{subfigure}{.5\textwidth}
  \centering
  \includegraphics[width=\linewidth]{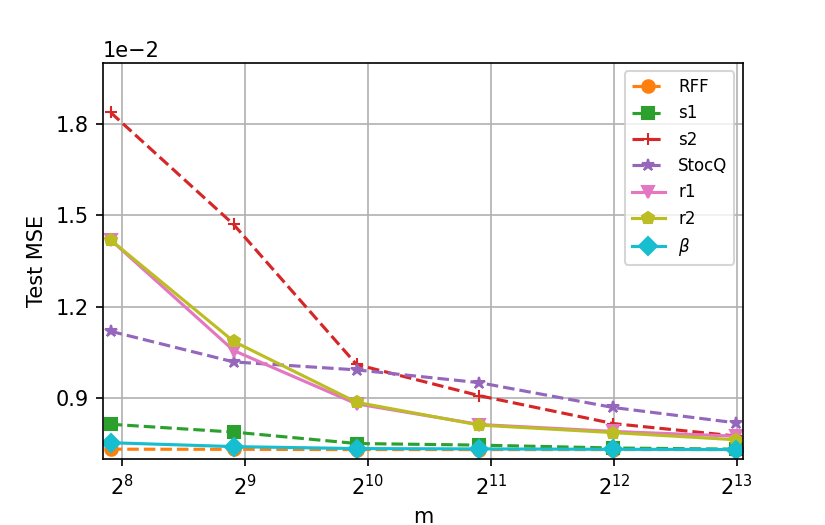}  
\end{subfigure}%
\begin{subfigure}{.5\textwidth}
  \centering
  \includegraphics[width=\linewidth]{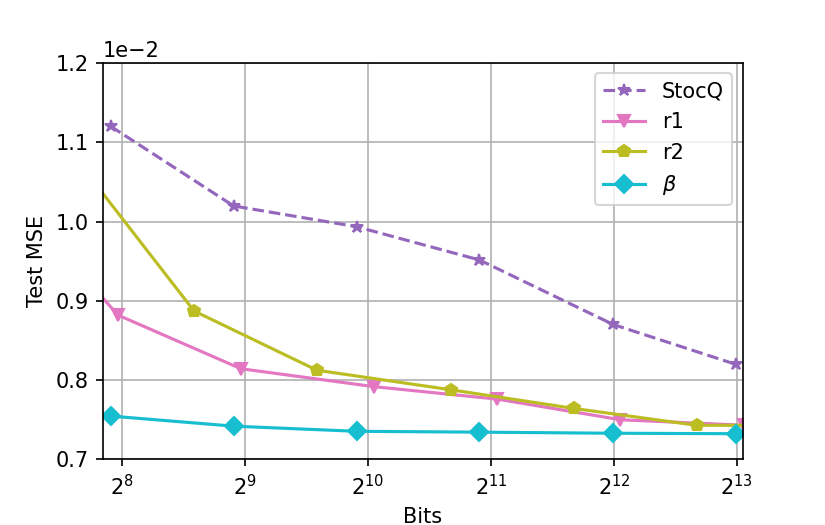}  
\end{subfigure}
\caption{Kernel ridge regression with $b=1$. The labels RFF, $s1$, $s2$, StocQ, $r1$, $r2$, $\beta$ represent $\widehat{k}_{\text{RFF}}$, $\widehat{k}_s$ for scenarios (1), (2), $\widehat{k}_{\text{StocQ}}$, $\widehat{k}_{\Sigma\Delta}^{(1)}$, $\widehat{k}_{\Sigma\Delta}^{(2)}$, and $\widehat{k}_{\beta}$ respectively.}
\label{fig:rkk_b1}
\end{figure}

\subsection{Kernel SVM}
To illustrate the performance of our methods for classification tasks, we perform Kernel SVM \cite{scholkopf2018learning, steinwart2008support} to evaluate different kernel approximations on the UCI ML hand-written digits dataset \cite{alpaydin1998cascading, xu1992methods}, in which $N=1797$ grayscale images compose $C=10$ classes and they are vectorized to $d=64$ dimensional vectors. Additionally, all pixel values are scaled in the range $[0,1]$ and we randomly split this dataset into $80\%$ for training and $20\%$ for testing. As for the classifier, we use the soft margin SVM with a regularization parameter $R=1$.

Note that in the binary classification case, i.e. labels $y_i\in\{-1,1\}$, our goal is to learn the coefficients $\alpha_i$, the intercept $b$, and the index set of support vectors $S$ in a decision function during the training stage:
\begin{equation}\label{svm-predictor}
g(x):=\sign\Bigl(\sum_{i\in\mathrm{S}}\alpha_i y_i k(x,x_i) +b\Bigr).
\end{equation}
Here, we use a RBF kernel $k(x,y)=\exp(-\gamma\|x-y\|_2^2)$ with $\gamma=1/(d\sigma_0^2)\approx 0.11$ and $\sigma_0^2$ being equal to the variance of training data. In the multi-class case, we implement the ``one-versus-one'' approach for multi-class classification where $\frac{C(C-1)}{2}$ classifiers are constructed and each one trains data from two classes. In our experiment, we found that a large embedding dimension $p=m/\lambda$ is needed and approximations $\widehat{k}_{\text{RFF}}$, $\widehat{k}_{\text{StocQ}}$, $\widehat{k}_{\Sigma\Delta}^{(1)}$ with $\lambda=2$, $\widehat{k}_{\Sigma\Delta}^{(2)}$ with $\lambda=3$, and $\widehat{k}_\beta$ with $\beta=1.1$, $\lambda=2$, are implemented for both training (obtaining $\alpha_i$, $b$, and $S$ in \eqref{svm-predictor}) and testing (predicting the class of an incoming sample $x$ by $g(x)$) phases, whereas the asymmetric scheme $\widehat{k}_s$ is only performed for inference with its parameters in \eqref{svm-predictor} learned from $\widehat{k}_{\text{RFF}}$ during the training stage. Moreover, as before there are two versions of $\widehat{k}_s$ used for making predictions:
\begin{enumerate}
    \item Keep the support vectors as unquantized RFFs and quantize the test point $x$, i.e. substitute $\widehat{k}_s(x_i, x)$ for $k(x,x_i)$ in \eqref{svm-predictor};
    \item Quantize the support vectors and leave the testing point $x$ as unquantized RFFs, i.e., replace $k(x,x_i)$ in \eqref{svm-predictor} with $\widehat{k}_s(x, x_i)$.
\end{enumerate}

\begin{figure}[ht]
\begin{subfigure}{.5\textwidth}
  \centering
  \includegraphics[width=\linewidth]{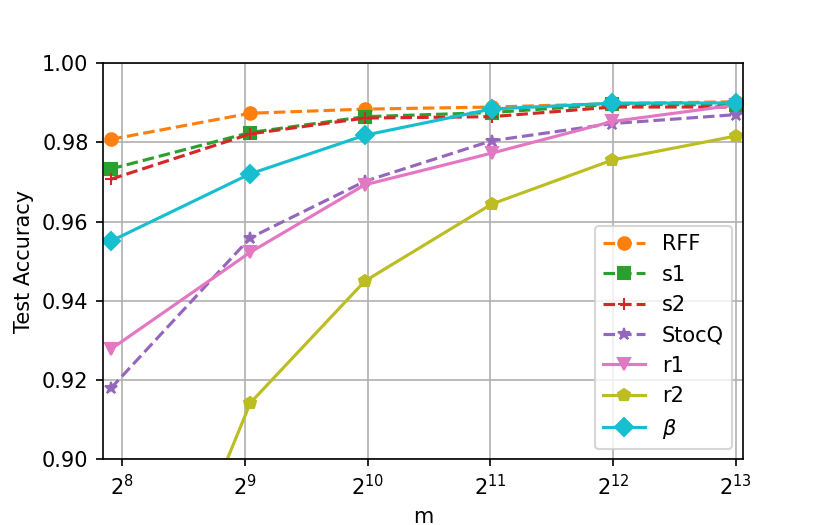}  
\end{subfigure}%
\begin{subfigure}{.5\textwidth}
  \centering
  \includegraphics[width=\linewidth]{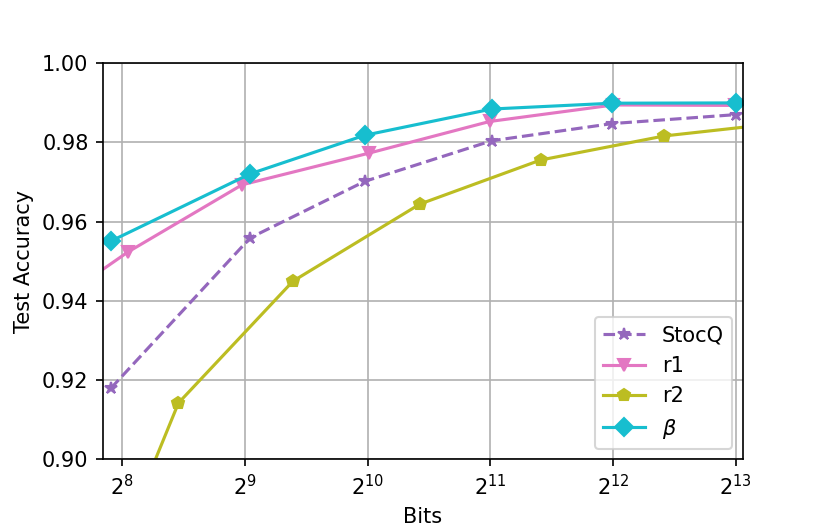}  
\end{subfigure}
\caption{Kernel SVM with $b=1$. The labels RFF, $s1$, $s2$, StocQ, $r1$, $r2$, $\beta$ represent $\widehat{k}_{\text{RFF}}$, $\widehat{k}_s$ for scenarios (1), (2), $\widehat{k}_{\text{StocQ}}$, $\widehat{k}_{\Sigma\Delta}^{(1)}$, $\widehat{k}_{\Sigma\Delta}^{(2)}$, and $\widehat{k}_{\beta}$ respectively. }
\label{fig:ksvm_b1}
\end{figure}
For each binary quantization scheme (with $b=1$), the average test accuracy over $30$ independent runs is plotted in Figure~\ref{fig:ksvm_b1}. We observe that, in regard to $m$, $Q_\beta$ substantially outperforms other fully-quantized schemes including $Q_{\Sigma\Delta}^{(r)}$ and $Q_{\text{StocQ}}$, but, as expected, it is still worse than the semi-quantized methods. Memory efficiency is characterized in the right plot by estimating the test accuracy against the storage cost (in terms of bits) per sample. Note that both $Q_\beta$ and $Q_{\Sigma\Delta}^{(1)}$ have significant advantage over the baseline method $Q_{\text{StocQ}}$, which means that our methods require less memory to achieve the same test accuracy when $b=1$. See Appendix~\ref{appendix:extra-experiments} for extra experiment results with $b=2,3$.

\subsection{Maximum Mean Discrepancy}
Given two distributions $p$ and $q$, and a kernel $k$ over $\mathcal{X}\subset\mathbb{R}^d$, the maximum mean discrepancy (MMD) has been shown to play an important role in the \emph{two-sample test} \cite{gretton2012kernel}, by proposing the null hypothesis $\mathcal{H}_0: p=q$ against the alternative hypothesis $\mathcal{H}_1: p\neq q$. The square of MMD distance can be computed by
\[
\mathrm{MMD}_k^2(p,q) = \E_{x,x^\prime}(k(x,x^\prime)) + \E_{y,y^\prime}(k(y,y^\prime)) -2\E_{x,y}(k(x,y))
\]
where $x,x^\prime\overset{iid}{\sim} p$ and $y,y^\prime\overset{iid}{\sim}q$. Here, we set $k$ to a RBF kernel, which is \emph{characteristic}  \cite{sriperumbudur2010hilbert} implying that $\mathrm{MMD}_k(p,q)$ is metric, i.e. $\mathrm{MMD}_k(p,q)=0\iff p=q$, and the following hypothesis test is consistent. 

\begin{figure}[ht]
\begin{subfigure}{.33\textwidth}
  \centering
  \includegraphics[width=\linewidth]{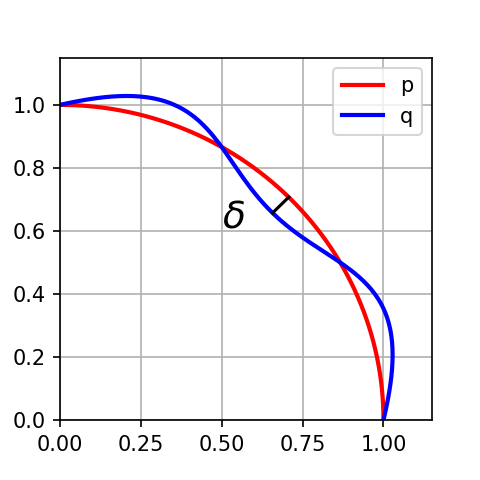}  
\caption{True distributions: $p$ and $q$}
\label{MMD_xy}
\end{subfigure}%
\begin{subfigure}{.33\textwidth}
  \centering
  \includegraphics[width=\linewidth]{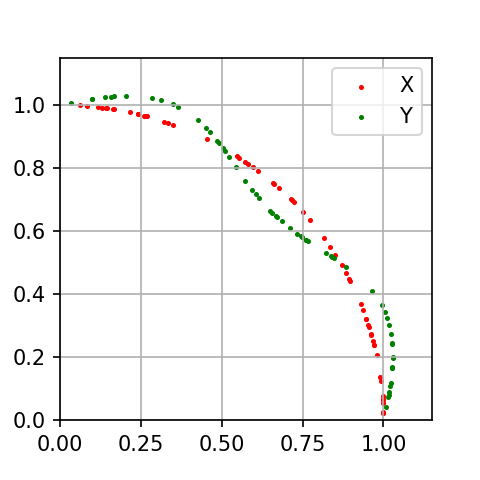}  
\caption{Samples of size $60$ from $p$, $q$} 
\end{subfigure}%
\begin{subfigure}{.33\textwidth}
  \centering
  \includegraphics[width=\linewidth]{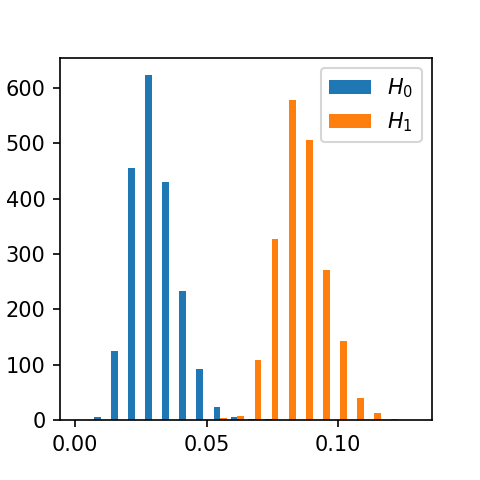}  
\caption{Counts vs MMD values} 
\label{MMD_xy_mmd}
\end{subfigure}
\caption{Two distributions and the MMD values based on the RBF kernel.}
\label{fig:mmd_xy}
\end{figure}

\begin{figure}[ht]
\begin{subfigure}{.5\textwidth}
  \centering
  \includegraphics[width=\linewidth]{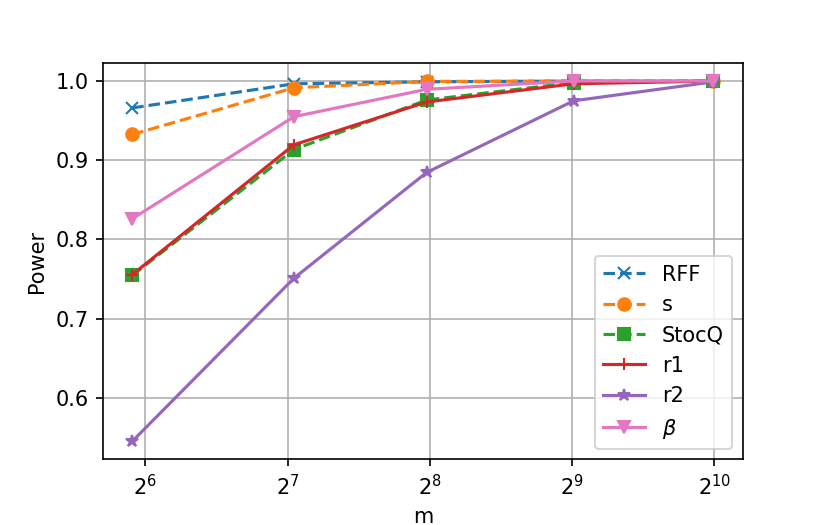}  
\end{subfigure}%
\begin{subfigure}{.5\textwidth}
  \centering
  \includegraphics[width=\linewidth]{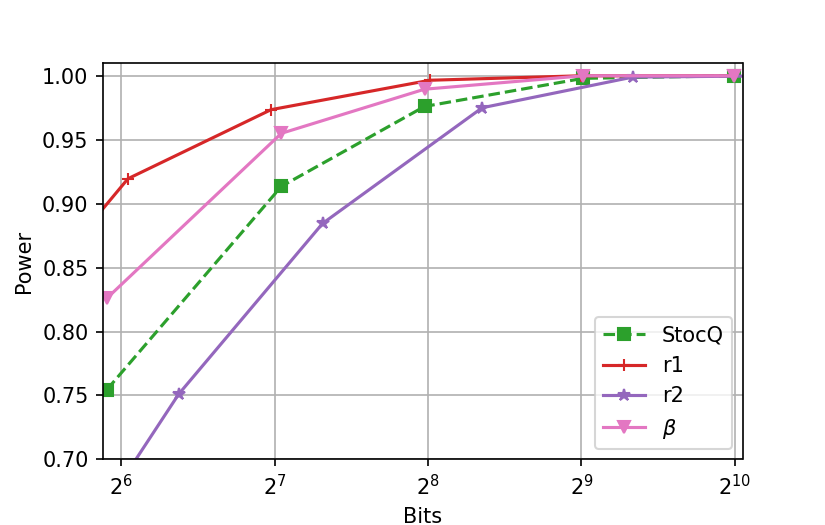}  
\end{subfigure}
\caption{Power of the permutation test with $b=1$. The labels RFF, $s$, StocQ, $r1$, $r2$, $\beta$ represent $\widehat{k}_{\text{RFF}}$, $\widehat{k}_s$, $\widehat{k}_{\text{StocQ}}$, $\widehat{k}_{\Sigma\Delta}^{(1)}$, $\widehat{k}_{\Sigma\Delta}^{(2)}$, and $\widehat{k}_{\beta}$ respectively. }
\label{fig:mmd_b1}
\end{figure}

In our experiment, the distribution $p$ is supported on a quadrant of the unit circle while $q$ is generated by perturbing $p$ by a gap of size $\delta$ at various regions, see Figure~\ref{MMD_xy}. Let $n=60$ and choose finite samples $X=\{x_1,\ldots,x_n\}\sim p$ and $Y=\{y_1,\ldots,y_n\}\sim q$. Then $\mathrm{MMD}_k(p,q)$ can be estimated by 
\begin{equation}\label{MMD_estimator}
\widehat{\mathrm{MMD}}_k^2(X,Y):=\frac{1}{n^2}\sum_{i,j=1}^n k(x_i, x_j) + \frac{1}{n^2}\sum_{i,j=1}^n k(y_i, y_j) - \frac{2}{n^2}\sum_{i,j=1}^n k(x_i, y_j).
\end{equation}
Under the null hypothesis $\mathcal{H}_0$, one can get the empirical distribution of \eqref{MMD_estimator} by reshuffling the data samples $X\cup Y$ many times ($t=2000$) and recomputing $\widehat{\mathrm{MMD}}_k^2(X^\prime,Y^\prime)$
on each partition $X^\prime\cup Y^\prime$. For a significance level of $\alpha=0.05$, $\mathcal{H}_0$ is rejected if the original $\widehat{\mathrm{MMD}}_k^2(X,Y)$ is greater than the $(1-\alpha)$ quantile from the empirical distribution. Figure~\ref{MMD_xy_mmd} shows that the empirical distributions of \eqref{MMD_estimator} under both $\mathcal{H}_0$ and $\mathcal{H}_1$ are separated well, where we use the ground truth RBF kernel with small bandwidth $\sigma=0.05$. 

In order to compare different quantization methods when $b=1$, we use the following approximations with optimal $\lambda$ to perform the permutation test: $\widehat{k}_{\text{RFF}}$, $\widehat{k}_{\text{StocQ}}$, $\widehat{k}_{\Sigma\Delta}^{(1)}$ with $\lambda=4$, $\widehat{k}_{\Sigma\Delta}^{(2)}$ with $\lambda=5$, and $\widehat{k}_\beta$ with $\beta=1.5$, $\lambda=4$. Due to the symmetry in \eqref{MMD_estimator}, $\widehat{k}_s$ can be implemented without worrying about the order of inputs. Additionally, if the probability of Type II error, i.e. false negative rate, is denoted by $\beta$, then the statistical power of our test is defined by
\[\mathrm{power}=1-\beta = \mathrm{P}(\mathrm{reject}\, \mathcal{H}_0|\,\mathcal{H}_1\mathrm{ is\,true})\]
In other words, the power equals to the portion of MMD values under $\mathcal{H}_1$ that are greater than the $(1-\alpha)$ quantile of MMD distribution under $\mathcal{H}_0$. In Figure~\ref{fig:mmd_b1}, we observe that, compared with other fully-quantized schemes, $Q_\beta$ has the greatest power in terms of $m$. The performance of semi-quantized scheme is pretty close to the plain RFF approximation while it requires more storage space, as discussed in Section~\ref{sec:main-result}. Moreover, Figure~\ref{fig:mmd_dist} presents the corresponding changes of the MMD distributions under $\mathcal{H}_0$ and $\mathcal{H}_1$, in which the overlap between the two distributions is considerably reduced as $m$ increases. Regarding the number of bits per sample, both $Q_{\Sigma\Delta}^{(1)}$ and $Q_\beta$ have remarkable advantage over $Q_{\text{StocQ}}$. Extra results related to $b=2,3$ can be found in Appendix~\ref{appendix:extra-experiments}. 

\begin{figure}[ht]
    \vspace{-60pt}
    \centering
    \includegraphics[width=\linewidth]{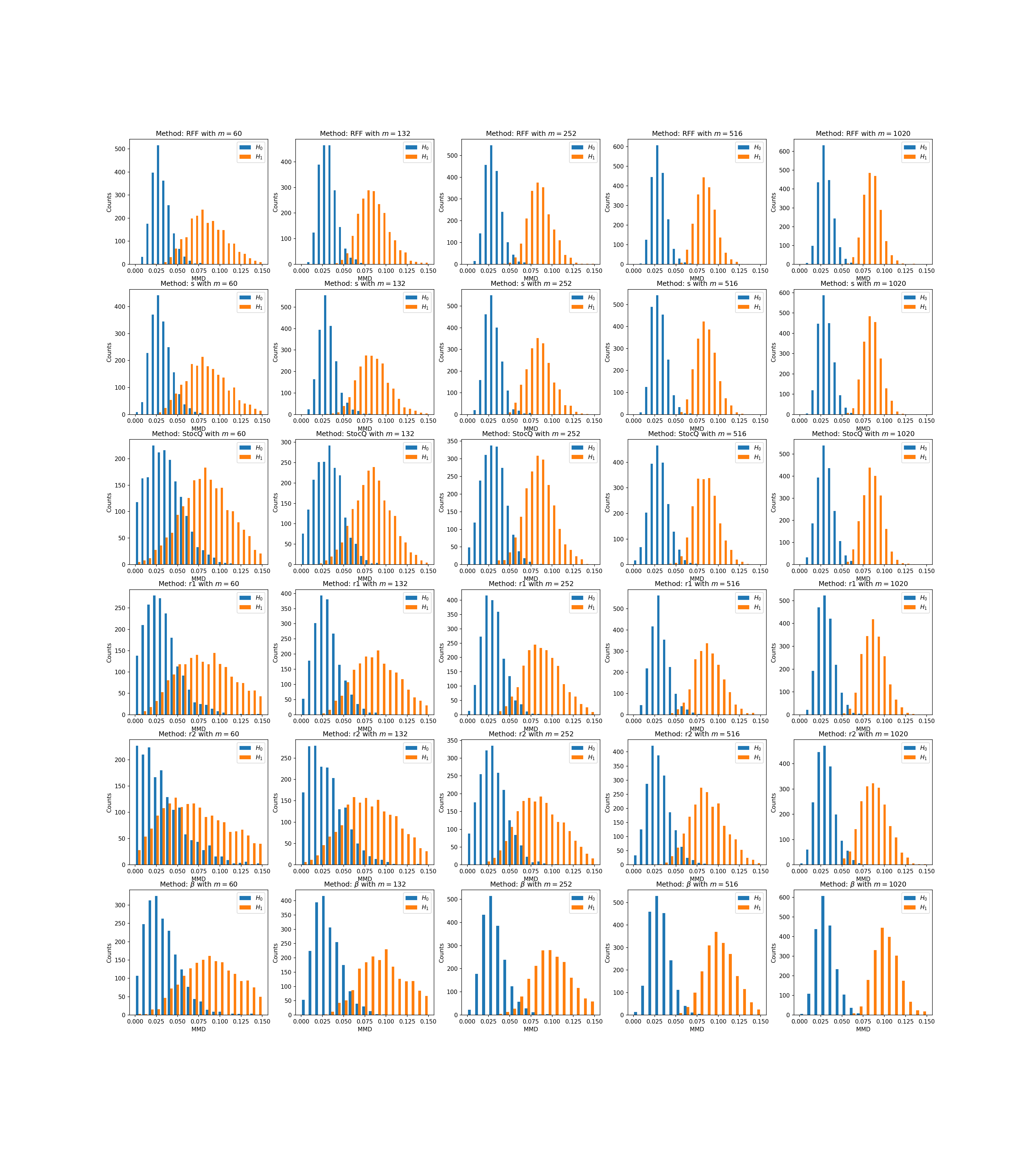}
    \vspace{-2.5cm}
    \caption{The empirical distributions of MMD values under $\mathcal{H}_0$ and $\mathcal{H}_1$.} 
    \label{fig:mmd_dist}
\end{figure}
\section{Conclusion}
In order to reduce memory requirement for training and storing kernel machines, we proposed a framework of using Sigma-Delta and distributed noise-shaping quantization schemes, $Q_{\Sigma\Delta}^{(r)}$ and $Q_\beta$, to approximate shift-invariant kernels. We have shown that these fully deterministic quantization schemes are capable of saving more bits than other baselines without compromising the performance.  
Importantly, we showed that, for all pairs of signals from an infinite low-complexity set, the approximations have uniform probabilistic error bounds yielding an exponential decay as the number of bits used increases. Empirically, we illustrated across popular kernel machines that the proposed quantization methods achieve strong performance both as a function of the dimension of the RFF embedding, and the number of bits used, especially in the case of binary embedding. 

\section*{Data Availability Statement}
The data underlying this article are available in the UCI Machine Learning Repository, at \url{https://archive.ics.uci.edu/ml/datasets/Optical+Recognition+of+Handwritten+Digits}

\section*{Funding} 
JZ was partially supported by grants NSF DMS 2012546 and 2012266.  
AC was partially supported by NSF DMS 1819222, 2012266. 
RS was partially supported by NSF DMS 2012546 and a UCSD senate research award. 

\bibliographystyle{abbrvnat}
\bibliography{citations}

\begin{thebibliography}{41}
\providecommand{\natexlab}[1]{#1}
\providecommand{\url}[1]{\texttt{#1}}
\expandafter\ifx\csname urlstyle\endcsname\relax
  \providecommand{\doi}[1]{doi: #1}\else
  \providecommand{\doi}{doi: \begingroup \urlstyle{rm}\Url}\fi

\bibitem[Agrawal et~al.(2019)Agrawal, Campbell, Huggins, and
  Broderick]{agrawal2019data}
R.~Agrawal, T.~Campbell, J.~Huggins, and T.~Broderick.
\newblock Data-dependent compression of random features for large-scale kernel
  approximation.
\newblock In \emph{The 22nd International Conference on Artificial Intelligence
  and Statistics}, pages 1822--1831. PMLR, 2019.

\bibitem[Alpaydin and Kaynak(1998)]{alpaydin1998cascading}
E.~Alpaydin and C.~Kaynak.
\newblock Cascading classifiers.
\newblock \emph{Kybernetika}, 34\penalty0 (4):\penalty0 369--374, 1998.

\bibitem[Avron et~al.(2017{\natexlab{a}})Avron, Clarkson, and
  Woodruff]{avron2017faster}
H.~Avron, K.~L. Clarkson, and D.~P. Woodruff.
\newblock Faster kernel ridge regression using sketching and preconditioning.
\newblock \emph{SIAM Journal on Matrix Analysis and Applications}, 38\penalty0
  (4):\penalty0 1116--1138, 2017{\natexlab{a}}.

\bibitem[Avron et~al.(2017{\natexlab{b}})Avron, Kapralov, Musco, Musco,
  Velingker, and Zandieh]{avron2017random}
H.~Avron, M.~Kapralov, C.~Musco, C.~Musco, A.~Velingker, and A.~Zandieh.
\newblock Random fourier features for kernel ridge regression: Approximation
  bounds and statistical guarantees.
\newblock In \emph{International Conference on Machine Learning}, pages
  253--262. PMLR, 2017{\natexlab{b}}.

\bibitem[Bach(2017)]{bach2017equivalence}
F.~Bach.
\newblock On the equivalence between kernel quadrature rules and random feature
  expansions.
\newblock \emph{The Journal of Machine Learning Research}, 18\penalty0
  (1):\penalty0 714--751, 2017.

\bibitem[Benedetto et~al.(2006{\natexlab{a}})Benedetto, Powell, and
  Y{\i}lmaz]{benedetto2006second}
J.~J. Benedetto, A.~M. Powell, and {\"O}.~Y{\i}lmaz.
\newblock Second-order sigma--delta ($\sigma$$\delta$) quantization of finite
  frame expansions.
\newblock \emph{Applied and Computational Harmonic Analysis}, 20\penalty0
  (1):\penalty0 126--148, 2006{\natexlab{a}}.

\bibitem[Benedetto et~al.(2006{\natexlab{b}})Benedetto, Powell, and
  Yilmaz]{benedetto2006sigma}
J.~J. Benedetto, A.~M. Powell, and O.~Yilmaz.
\newblock Sigma-delta quantization and finite frames.
\newblock \emph{IEEE Transactions on Information Theory}, 52\penalty0
  (5):\penalty0 1990--2005, 2006{\natexlab{b}}.

\bibitem[Boufounos and Rane(2013)]{boufounos2013efficient}
P.~T. Boufounos and S.~Rane.
\newblock Efficient coding of signal distances using universal quantized
  embeddings.
\newblock In \emph{DCC}, pages 251--260, 2013.

\bibitem[Chou and G{\"u}nt{\"u}rk(2016)]{chou2016distributed}
E.~Chou and C.~S. G{\"u}nt{\"u}rk.
\newblock Distributed noise-shaping quantization: I. beta duals of finite
  frames and near-optimal quantization of random measurements.
\newblock \emph{Constructive Approximation}, 44\penalty0 (1):\penalty0 1--22,
  2016.

\bibitem[Chou and G{\"u}nt{\"u}rk(2017)]{chou2017distributed}
E.~Chou and C.~S. G{\"u}nt{\"u}rk.
\newblock Distributed noise-shaping quantization: Ii. classical frames.
\newblock In \emph{Excursions in Harmonic Analysis, Volume 5}, pages 179--198.
  Springer, 2017.

\bibitem[Chou et~al.(2015)Chou, G{\"u}nt{\"u}rk, Krahmer, Saab, and
  Y{\i}lmaz]{chou2015noise}
E.~Chou, C.~S. G{\"u}nt{\"u}rk, F.~Krahmer, R.~Saab, and {\"O}.~Y{\i}lmaz.
\newblock Noise-shaping quantization methods for frame-based and compressive
  sampling systems.
\newblock \emph{Sampling theory, a renaissance}, pages 157--184, 2015.

\bibitem[Cucker and Smale(2002)]{Cucker}
F.~Cucker and S.~Smale.
\newblock On the mathematical foundations of learning.
\newblock \emph{Bulletin of the American mathematical society}, 39\penalty0
  (1):\penalty0 1--49, 2002.

\bibitem[Danzer(1963)]{danzer1963helly}
L.~Danzer.
\newblock " helly's theorem and its relatives," in convexity.
\newblock In \emph{Proc. Symp. Pure Math.}, volume~7, pages 101--180. Amer.
  Math. Soc., 1963.

\bibitem[Daubechies and DeVore(2003)]{daubechies2003approximating}
I.~Daubechies and R.~DeVore.
\newblock Approximating a bandlimited function using very coarsely quantized
  data: A family of stable sigma-delta modulators of arbitrary order.
\newblock \emph{Annals of mathematics}, pages 679--710, 2003.

\bibitem[Deift et~al.(2011)Deift, Krahmer, and
  G{\"u}nt{\"u}rk]{deift2011optimal}
P.~Deift, F.~Krahmer, and C.~S. G{\"u}nt{\"u}rk.
\newblock An optimal family of exponentially accurate one-bit sigma-delta
  quantization schemes.
\newblock \emph{Communications on Pure and Applied Mathematics}, 64\penalty0
  (7):\penalty0 883--919, 2011.

\bibitem[Foucart and Rauhut(2013)]{foucartmathematical}
S.~Foucart and H.~Rauhut.
\newblock \emph{A Mathematical Introduction to Compressive Sensing}.
\newblock Springer, 2013.

\bibitem[Gretton et~al.(2012)Gretton, Borgwardt, Rasch, Sch{\"o}lkopf, and
  Smola]{gretton2012kernel}
A.~Gretton, K.~M. Borgwardt, M.~J. Rasch, B.~Sch{\"o}lkopf, and A.~Smola.
\newblock A kernel two-sample test.
\newblock \emph{The Journal of Machine Learning Research}, 13\penalty0
  (1):\penalty0 723--773, 2012.

\bibitem[G{\"u}nt{\"u}rk(2003)]{gunturk2003one}
C.~S. G{\"u}nt{\"u}rk.
\newblock One-bit sigma-delta quantization with exponential accuracy.
\newblock \emph{Communications on Pure and Applied Mathematics: A Journal
  Issued by the Courant Institute of Mathematical Sciences}, 56\penalty0
  (11):\penalty0 1608--1630, 2003.

\bibitem[G{\"u}nt{\"u}rk et~al.(2013)G{\"u}nt{\"u}rk, Lammers, Powell, Saab,
  and Y{\i}lmaz]{gunturk2013sobolev}
C.~S. G{\"u}nt{\"u}rk, M.~Lammers, A.~M. Powell, R.~Saab, and {\"O}.~Y{\i}lmaz.
\newblock Sobolev duals for random frames and $\sigma$$\delta$ quantization of
  compressed sensing measurements.
\newblock \emph{Foundations of Computational mathematics}, 13\penalty0
  (1):\penalty0 1--36, 2013.

\bibitem[Huynh(2016)]{huynh2016accurate}
T.~Huynh.
\newblock \emph{Accurate quantization in redundant systems: From frames to
  compressive sampling and phase retrieval}.
\newblock PhD thesis, New York University, 2016.

\bibitem[Huynh and Saab(2020)]{huynh2020fast}
T.~Huynh and R.~Saab.
\newblock Fast binary embeddings and quantized compressed sensing with
  structured matrices.
\newblock \emph{Communications on Pure and Applied Mathematics}, 73\penalty0
  (1):\penalty0 110--149, 2020.

\bibitem[Krahmer et~al.(2012)Krahmer, Saab, and Ward]{krahmer2012root}
F.~Krahmer, R.~Saab, and R.~Ward.
\newblock Root-exponential accuracy for coarse quantization of finite frame
  expansions.
\newblock \emph{IEEE transactions on information theory}, 58\penalty0
  (2):\penalty0 1069--1079, 2012.

\bibitem[Li and Li(2021)]{li2021quantization}
X.~Li and P.~Li.
\newblock Quantization algorithms for random fourier features.
\newblock \emph{arXiv preprint arXiv:2102.13079}, 2021.

\bibitem[Lin(2007)]{lin2007large}
C.-J. Lin.
\newblock \emph{Large-scale kernel machines}.
\newblock MIT press, 2007.

\bibitem[Liu et~al.(2020)Liu, Huang, Chen, and Suykens]{liu2020random}
F.~Liu, X.~Huang, Y.~Chen, and J.~A. Suykens.
\newblock Random features for kernel approximation: A survey in algorithms,
  theory, and beyond.
\newblock \emph{arXiv preprint arXiv:2004.11154}, 2020.

\bibitem[Loomis(2013)]{loomis2013introduction}
L.~H. Loomis.
\newblock \emph{Introduction to abstract harmonic analysis}.
\newblock Courier Corporation, 2013.

\bibitem[May et~al.(2019)May, Garakani, Lu, Guo, Liu, Bellet, Fan, Collins,
  Hsu, Kingsbury, Picheny, and Sha]{JMLR:v20:17-026}
A.~May, A.~B. Garakani, Z.~Lu, D.~Guo, K.~Liu, A.~Bellet, L.~Fan, M.~Collins,
  D.~Hsu, B.~Kingsbury, M.~Picheny, and F.~Sha.
\newblock Kernel approximation methods for speech recognition.
\newblock \emph{Journal of Machine Learning Research}, 20\penalty0
  (59):\penalty0 1--36, 2019.
\newblock URL \url{http://jmlr.org/papers/v20/17-026.html}.

\bibitem[Murphy(2012)]{murphy2012machine}
K.~P. Murphy.
\newblock \emph{Machine learning: a probabilistic perspective}.
\newblock MIT press, 2012.

\bibitem[Rahimi and Recht(2007)]{Rahimi}
A.~Rahimi and B.~Recht.
\newblock Random features for large-scale kernel machines.
\newblock \emph{Advances in neural information processing systems},
  20:\penalty0 1177--1184, 2007.

\bibitem[Rudi and Rosasco(2017)]{rudi2017generalization}
A.~Rudi and L.~Rosasco.
\newblock Generalization properties of learning with random features.
\newblock In \emph{NIPS}, pages 3215--3225, 2017.

\bibitem[Schellekens and Jacques(2020)]{schellekens2020breaking}
V.~Schellekens and L.~Jacques.
\newblock Breaking the waves: asymmetric random periodic features for
  low-bitrate kernel machines.
\newblock \emph{arXiv preprint arXiv:2004.06560}, 2020.

\bibitem[Scholkopf and Smola(2018)]{scholkopf2018learning}
B.~Scholkopf and A.~J. Smola.
\newblock \emph{Learning with kernels: support vector machines, regularization,
  optimization, and beyond}.
\newblock Adaptive Computation and Machine Learning series, 2018.

\bibitem[Shawe-Taylor et~al.(2004)Shawe-Taylor, Cristianini,
  et~al.]{shawe2004kernel}
J.~Shawe-Taylor, N.~Cristianini, et~al.
\newblock \emph{Kernel methods for pattern analysis}.
\newblock Cambridge university press, 2004.

\bibitem[Sriperumbudur and Szab{\'o}(2015)]{sriperumbudur2015optimal}
B.~K. Sriperumbudur and Z.~Szab{\'o}.
\newblock Optimal rates for random fourier features.
\newblock In \emph{Proceedings of the 28th International Conference on Neural
  Information Processing Systems-Volume 1}, pages 1144--1152, 2015.

\bibitem[Sriperumbudur et~al.(2010)Sriperumbudur, Gretton, Fukumizu,
  Sch{\"o}lkopf, and Lanckriet]{sriperumbudur2010hilbert}
B.~K. Sriperumbudur, A.~Gretton, K.~Fukumizu, B.~Sch{\"o}lkopf, and G.~R.
  Lanckriet.
\newblock Hilbert space embeddings and metrics on probability measures.
\newblock \emph{The Journal of Machine Learning Research}, 11:\penalty0
  1517--1561, 2010.

\bibitem[Steinwart and Christmann(2008)]{steinwart2008support}
I.~Steinwart and A.~Christmann.
\newblock \emph{Support vector machines}.
\newblock Springer Science \& Business Media, 2008.

\bibitem[Sutherland and Schneider(2015)]{sutherland2015error}
D.~J. Sutherland and J.~Schneider.
\newblock On the error of random fourier features.
\newblock In \emph{Proceedings of the Thirty-First Conference on Uncertainty in
  Artificial Intelligence}, pages 862--871, 2015.

\bibitem[Tu et~al.(2016)Tu, Roelofs, Venkataraman, and Recht]{tu2016large}
S.~Tu, R.~Roelofs, S.~Venkataraman, and B.~Recht.
\newblock Large scale kernel learning using block coordinate descent.
\newblock \emph{arXiv preprint arXiv:1602.05310}, 2016.

\bibitem[Xu et~al.(1992)Xu, Krzyzak, and Suen]{xu1992methods}
L.~Xu, A.~Krzyzak, and C.~Y. Suen.
\newblock Methods of combining multiple classifiers and their applications to
  handwriting recognition.
\newblock \emph{IEEE transactions on systems, man, and cybernetics},
  22\penalty0 (3):\penalty0 418--435, 1992.

\bibitem[Zhang and Saab(2021)]{zhang2021faster}
J.~Zhang and R.~Saab.
\newblock Faster binary embeddings for preserving euclidean distances.
\newblock In \emph{International Conference on Learning Representations}, 2021.
\newblock URL \url{https://openreview.net/forum?id=YCXrx6rRCXO}.

\bibitem[Zhang et~al.(2019)Zhang, May, Dao, and R{\'e}]{zhang2019low}
J.~Zhang, A.~May, T.~Dao, and C.~R{\'e}.
\newblock Low-precision random fourier features for memory-constrained kernel
  approximation.
\newblock In \emph{The 22nd International Conference on Artificial Intelligence
  and Statistics}, pages 1264--1274. PMLR, 2019.

\end{thebibliography}

\appendix
\section{Stable Quantization Methods}\label{appendix:quantization}
\noindent \textbf{The general definition for stable $Q_{\Sigma\Delta}^{(r)}$.}
Although it is a non-trivial task to design a stable $Q_{\Sigma\Delta}^{(r)}$ for $r>1$, families of $\Sigma\Delta$ quantization schemes that achieve this goal have been designed \cite{daubechies2003approximating, deift2011optimal,gunturk2003one}, and we adopt the version in \cite{deift2011optimal}. Specifically, an $r$-th order $\Sigma\Delta$ quantization scheme may also arise from the following difference equation
\begin{equation}\label{stable-quantizer}
y-q= H * v
\end{equation} 
where $*$ is the convolution operator and the sequence $H:=D^r g$ with $g\in\ell^1$. Then any bounded solution $v$ of \eqref{stable-quantizer} gives rise to a bounded solution $u$ of \eqref{diff-sigma-delta} via $u=g*v$. By change of variables, \eqref{diff-sigma-delta} can be reformulated as \eqref{stable-quantizer}. By choosing a proper filter $h:=\delta^{(0)}-H$, where $\delta^{(0)}$ denotes the Kronecker delta sequence supported at $0$, one can implement \eqref{stable-quantizer} by $v_i=(h*v)_i+y_i-q_i$ and the corresponding stable quantization scheme $Q_{\Sigma\Delta}^{(r)}$ reads as
\begin{equation}\label{general-quantizer-new}
\begin{cases}
q_i=Q((h*v)_i+y_i), \\
v_i=(h*v)_i+y_i-q_i.
\end{cases}
\end{equation}
Furthermore, the above design leads to the following result from \cite{deift2011optimal, krahmer2012root}, which exploits the constant $c(K,\mu, r)$ to bound $\|u\|_\infty$.
\begin{proposition}\label{prop:stability-sigma-delta} 
There exists a universal constant $C>0$ such that the $\Sigma\Delta$ schemes \eqref{sigma-delta-eq} and \eqref{general-quantizer-new} with alphabet $\mathcal{A}$ in \eqref{alphabet}, are stable, and 
\[
	\|y\|_\infty \leq \mu < 1  \Longrightarrow \|u\|_\infty  \leq c(K,r) := \frac{CC_1^r r^r}{2K-1},
\]
where $C_1=\bigl( \bigl\lceil \frac{\pi^2}{(\cosh^{-1}\gamma)^2}\bigr\rceil \frac{e}{\pi} \bigr)$ with $\gamma:=2K-(2K-1)\mu$.
\end{proposition}
Note that even with the $b=1$ bit alphabet, i.e. $K=1$ and $\mathcal{A}=\{-1, 1\}$, stability can be guaranteed with  
\[
\|y\|_\infty \leq \mu <1 \quad \Longrightarrow\quad \|u\|_\infty \leq C\cdot C_1^r\cdot r^r.
\]
\textbf{The stability of $Q_\beta$.} The relevant result for stability of the noise-shaping quantization schemes \eqref{beta-eq} is the following proposition, which can be simply proved by induction or can be found in \cite{chou2017distributed}. 

\begin{proposition}\label{prop:stability-beta}
The noise-shaping scheme \eqref{beta-eq} with alphabet $\mathcal{A}$ in \eqref{alphabet} is stable and 
\[
	\|y\|_\infty \leq \frac{2K-\beta}{2K-1}  \Longrightarrow \|u\|_\infty \leq c(K,\beta):=\frac{1}{2K-1}.
\]
\end{proposition}

\section{A comparison of kernel approximations}\label{appendix:approx-comparison}
In Figure~\ref{fig:kernel-approx}, we evaluate approximated kernels \eqref{approx-sigma-delta} and \eqref{approx-beta} in Section~\ref{sec:intro} 
on $n=1000$ pairs of points $\{x_i, y_i\}_{i=1}^n$ in $\mathbb{R}^d$ with $d=50$ such that for each $i$
\[
x_i\sim\mathcal{N}(0, I_d),\quad u_i\sim\mathcal{N}(0, I_d), \quad y_i = x_i + \frac{5i}{n}\cdot\frac{u_i}{\|u_i\|_2}.
\]
Moreover, each data point $x_i$ is represented by $m=3000$ RFF features and we use $3$-bit quantizers to guarantee good performance for all methods. The target RBF kernel (red curves) is $k(x,y)=\exp(-\|x-y\|_2^2/2\sigma^2)$ with $\gamma:=1/2\sigma^2=\frac{1}{5}$ and note that the approximations (black dots) have their $\ell_2$ distances $\|x-y\|_2$ uniformly distributed in the range $[0,5]$. We see that both $\widehat{k}_{\Sigma\Delta}^{(r)}$ and $\widehat{k}_\beta$ can approximate $k$ well.

\begin{figure}[H]
\vspace{-20pt}
\begin{subfigure}{.33\textwidth}
  \centering
  \includegraphics[width=\linewidth]{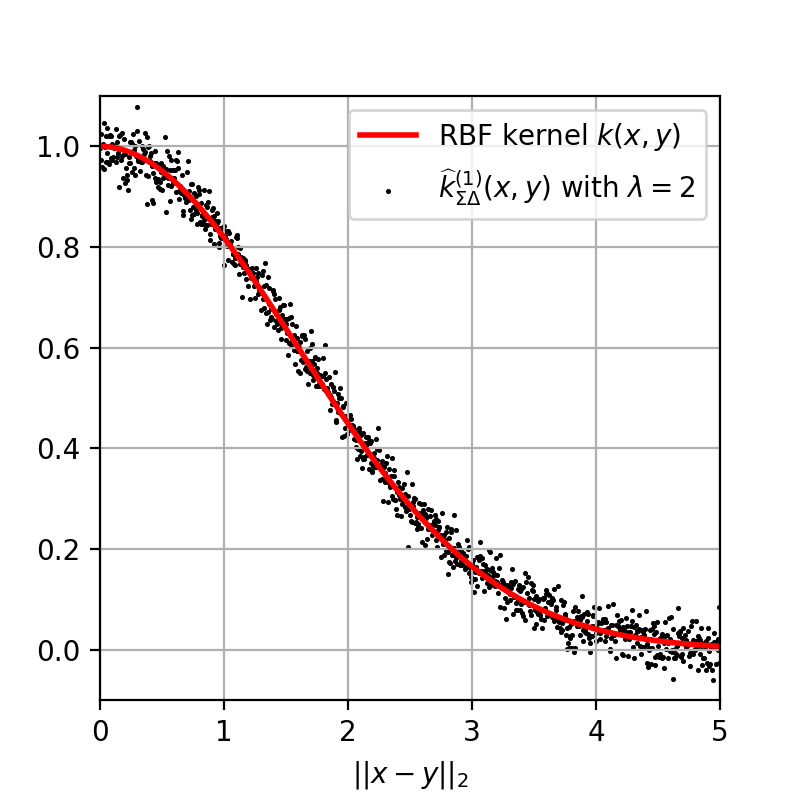}  
  \caption{$\widehat{k}^{(r)}_{\Sigma\Delta}(x,y)$ with $r=1$} 
  \label{fig:sub-fourth}
\end{subfigure}%
\begin{subfigure}{.33\textwidth}
  \centering
  \includegraphics[width=\linewidth]{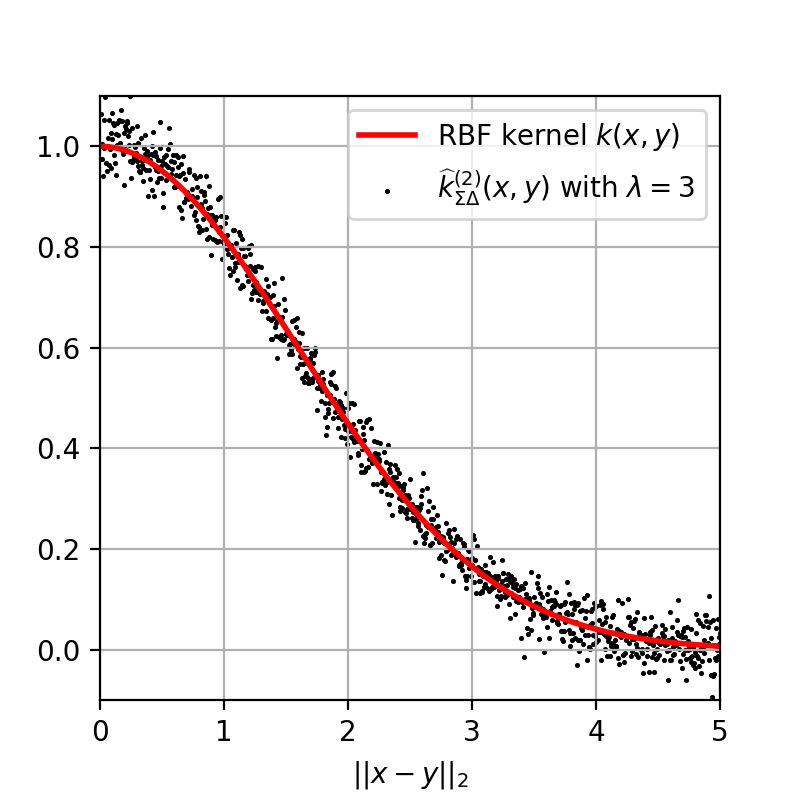}  
  \caption{$\widehat{k}^{(r)}_{\Sigma\Delta}(x,y)$ with $r=2$} 
  \label{fig:sub-fifth}
\end{subfigure}%
\begin{subfigure}{.33\textwidth}
  \centering
  \includegraphics[width=\linewidth]{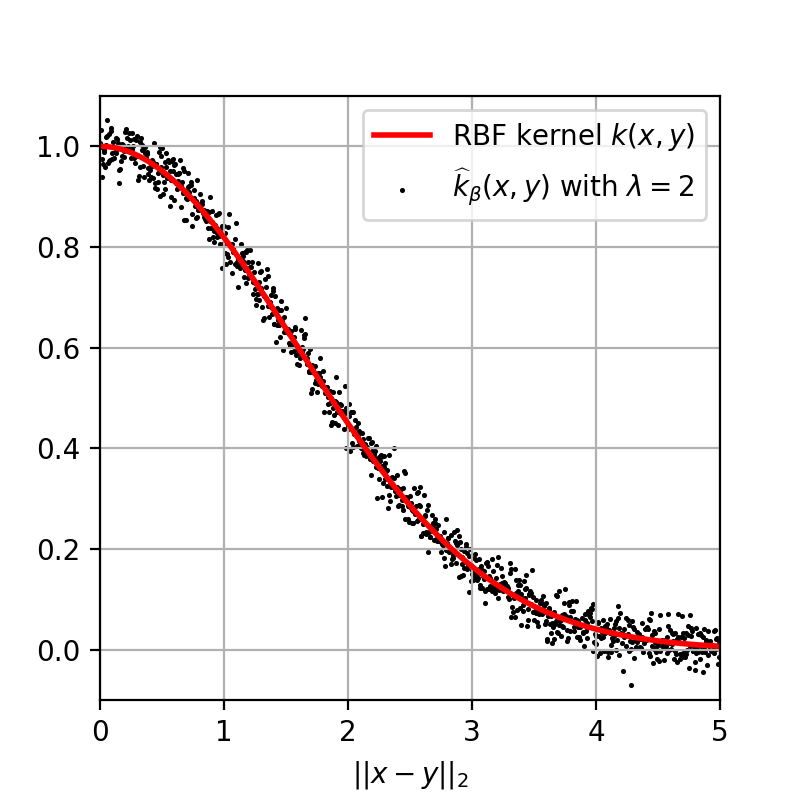}  
  \caption{$\widehat{k}_\beta(x,y)$ with $\beta=1.1$}. 
  \label{fig:sub-sixth}
\end{subfigure}
\caption{Kernel Approximations with $b=3$.}
\label{fig:kernel-approx}
\end{figure}

\section{More Figures in Section~\ref{sec:experiment}}\label{appendix:extra-experiments} 
\noindent\textbf{KRR.}
In Figure~\ref{fig:rkk_b2} and Figure~\ref{fig:rkk_b3}, we show the KRR results for $b=2$ and $b=3$ respectively. Same as in the Section~\ref{sec:experiment}, we see that the proposed methods $Q_{\Sigma\Delta}^{(r)}$ and $Q_\beta$ have strong performance in terms of $m$ and the number of bits used for each sample.
\begin{figure}[ht]
\vspace{-15pt}
\begin{subfigure}{.5\textwidth}
  \centering
  \includegraphics[width=\linewidth]{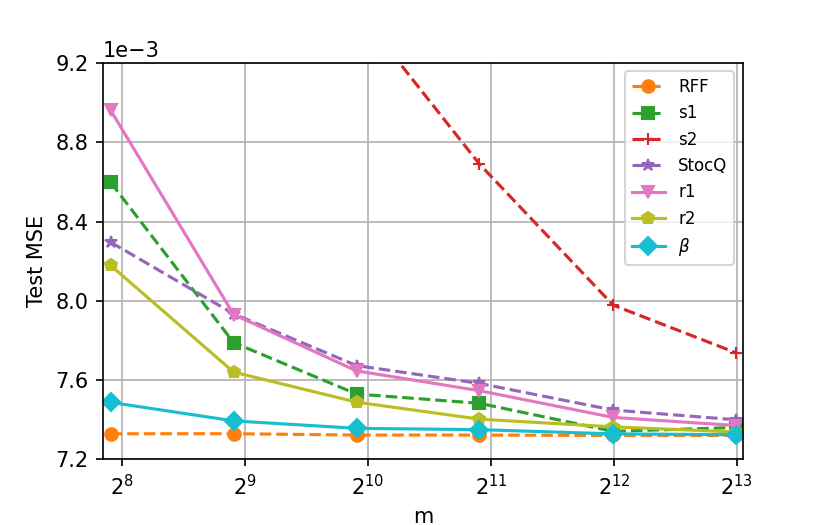}  
\end{subfigure}%
\begin{subfigure}{.5\textwidth}
  \centering
  \includegraphics[width=\linewidth]{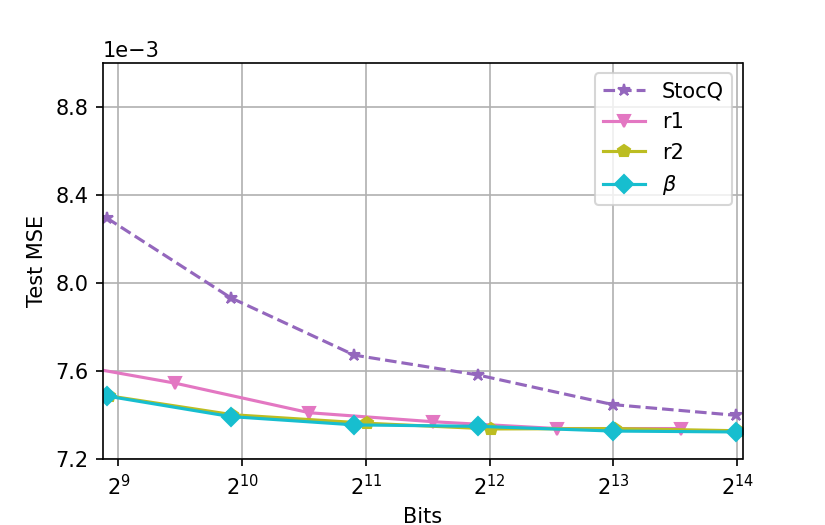}  
\end{subfigure}
\caption{Kernel ridge regression with $b=2$.}
\label{fig:rkk_b2}
\vspace{-10pt}
\end{figure}

\begin{figure}[ht]
\vspace{-15pt}
\begin{subfigure}{.5\textwidth}
  \centering
  \includegraphics[width=\linewidth]{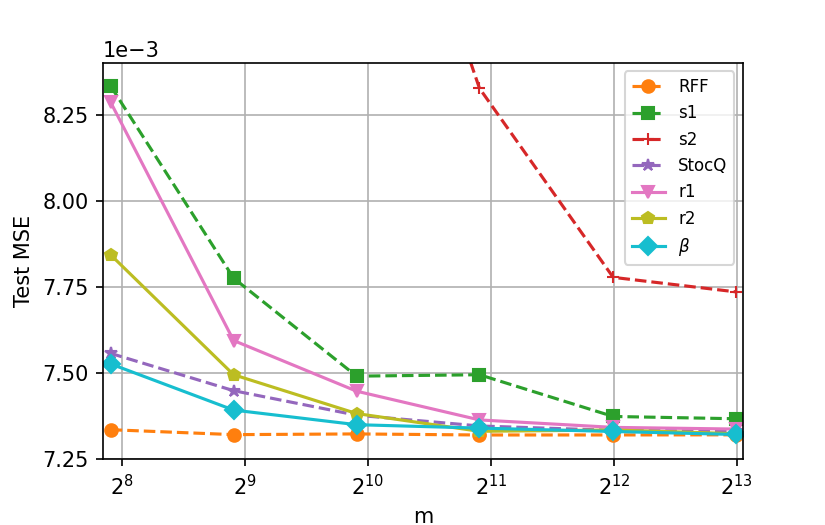}  
\end{subfigure}%
\begin{subfigure}{.5\textwidth}
  \centering
  \includegraphics[width=\linewidth]{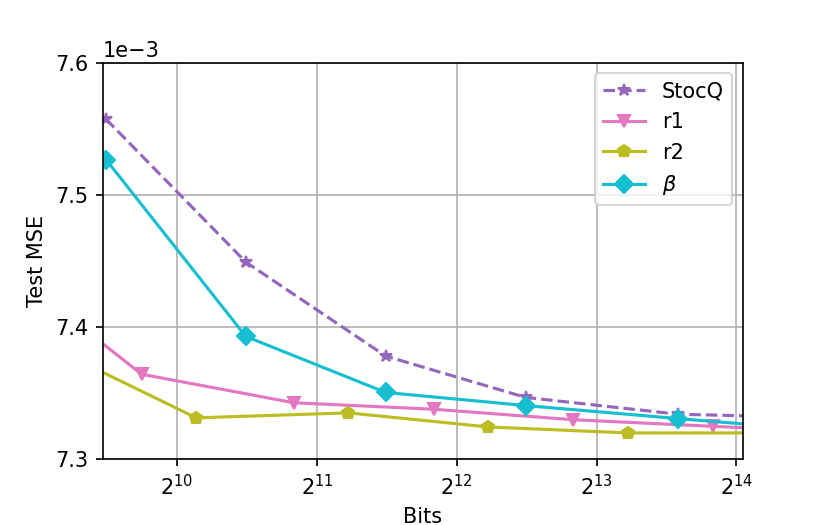}  
\end{subfigure}
\caption{Kernel ridge regression with $b=3$.}
\label{fig:rkk_b3}
\vspace{-10pt}
\end{figure}

\noindent\textbf{Kernel SVM.}
Figure~\ref{fig:ksvm_b2} and Figure~\ref{fig:ksvm_b3} illustrate the performance of kernel SVM with $b=2$ and $b=3$ respectively. As we expect, the gap across various schemes shrinks when we use multibit quantizers, where $Q_{\text{StocQ}}$ is comparable with $Q_{\Sigma\Delta}^{(r)}$ and $Q_\beta$.

\begin{figure}[ht]
\vspace{-15pt}
\begin{subfigure}{.5\textwidth}
  \centering
  \includegraphics[width=\linewidth]{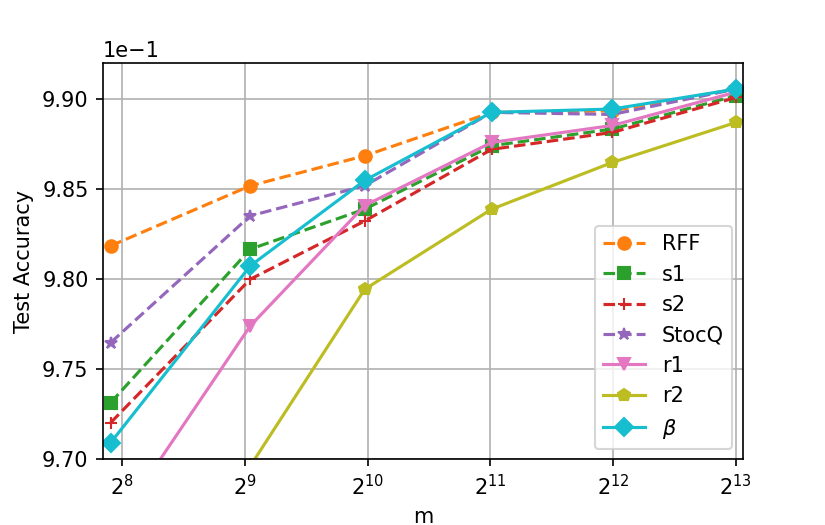}  
\end{subfigure}%
\begin{subfigure}{.5\textwidth}
  \centering
  \includegraphics[width=\linewidth]{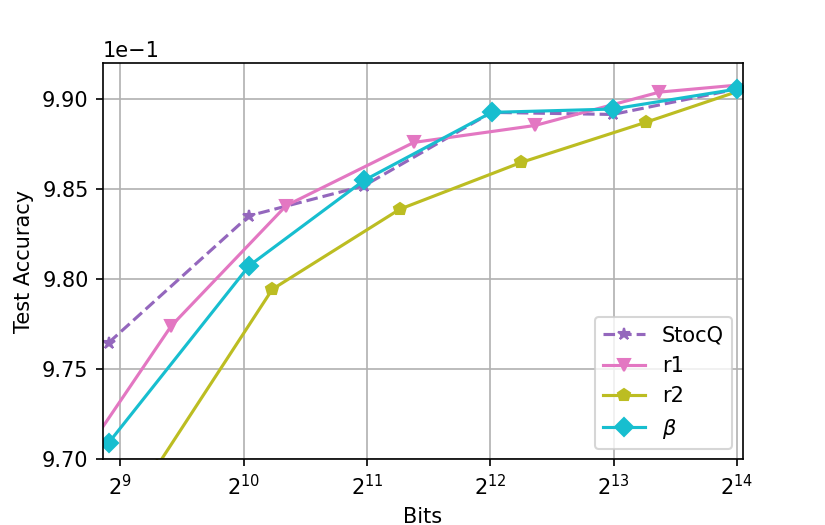}  
\end{subfigure}
\caption{Kernel SVM with $b=2$.}
\label{fig:ksvm_b2}
\vspace{-10pt}
\end{figure}

\begin{figure}[ht]
\vspace{-15pt}
\begin{subfigure}{.5\textwidth}
  \centering
  \includegraphics[width=\linewidth]{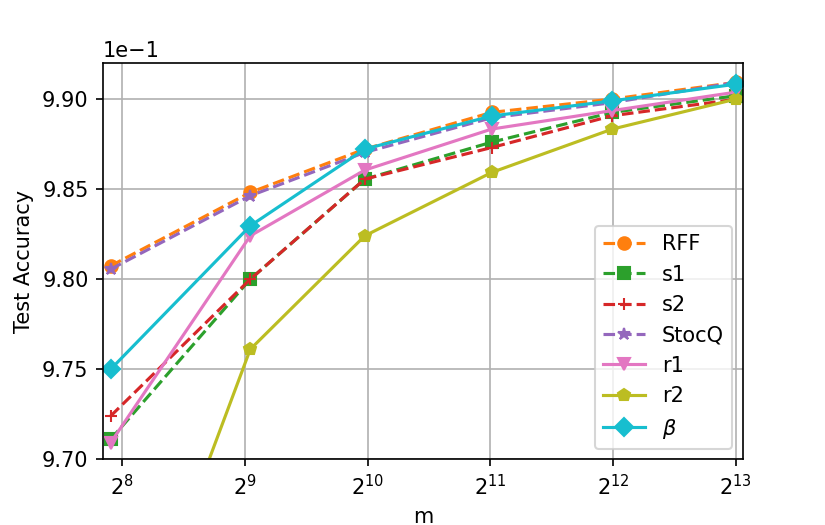}  
\end{subfigure}%
\begin{subfigure}{.5\textwidth}
  \centering
  \includegraphics[width=\linewidth]{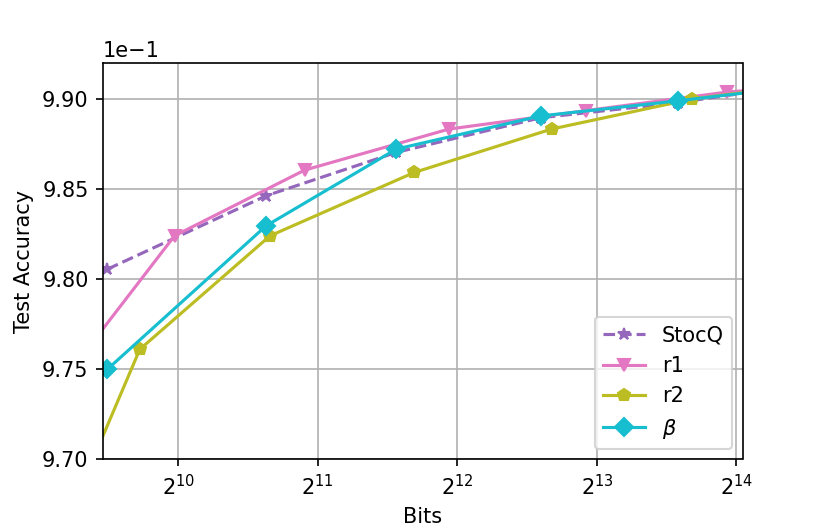}  
\end{subfigure}
\caption{Kernel SVM with $b=3$.}
\label{fig:ksvm_b3}
\vspace{-10pt}
\end{figure}

\noindent\textbf{MMD.}
As a continuation of the two-sample test in Section~\ref{sec:experiment}, Figure~\ref{fig:mmd_b2} and Figure~\ref{fig:mmd_b3} imply that both semi-quantized scheme and $Q_{\text{StocQ}}$ have better performance with respect to $m$, while $Q_{\Sigma\Delta}^{(r)}$ can save more memory than other quantization methods.

\begin{figure}[ht]
\vspace{-10pt}
\begin{subfigure}{.5\textwidth}
  \centering
  \includegraphics[width=\linewidth]{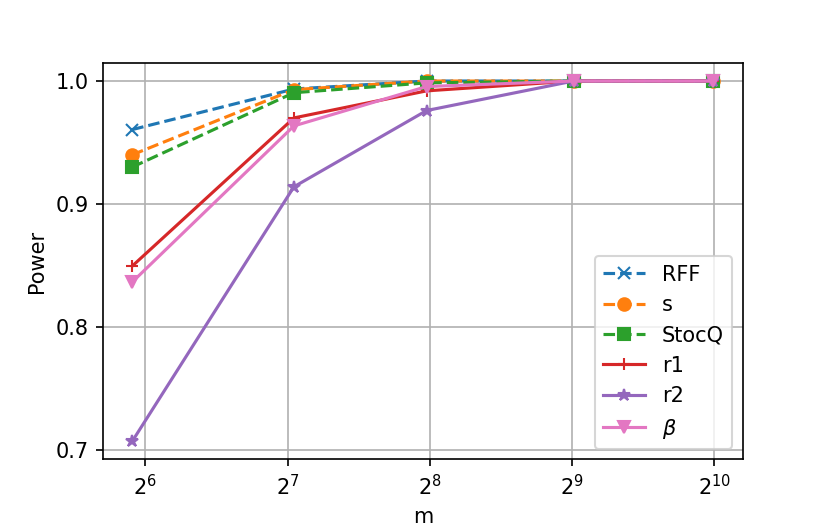}  
\end{subfigure}%
\begin{subfigure}{.5\textwidth}
  \centering
  \includegraphics[width=\linewidth]{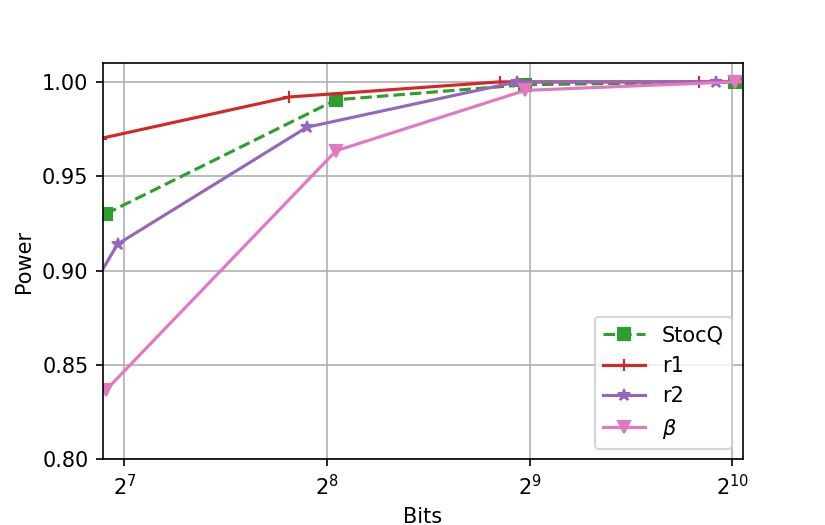}  
\end{subfigure}
\caption{Power of the permutation test with $b=2$.}
\label{fig:mmd_b2}
\vspace{-10pt}
\end{figure}

\begin{figure}[ht]
\vspace{-25pt}
\begin{subfigure}{.5\textwidth}
  \centering
  \includegraphics[width=\linewidth]{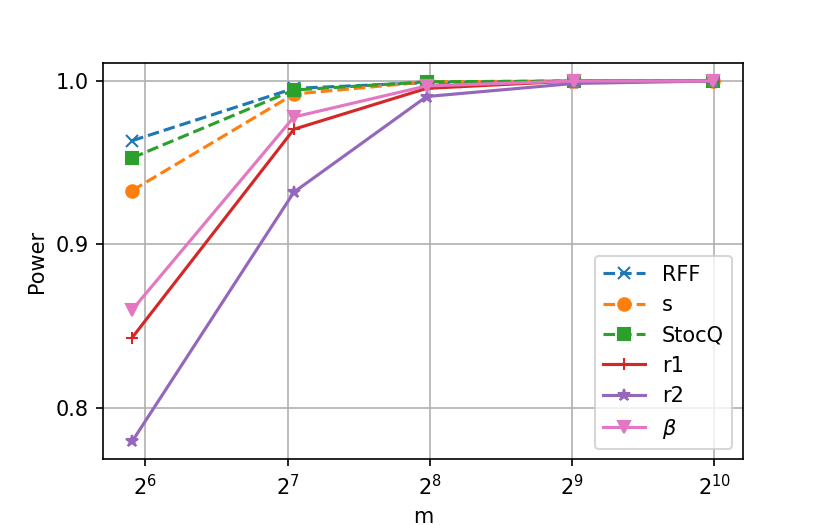}  
\end{subfigure}%
\begin{subfigure}{.5\textwidth}
  \centering
  \includegraphics[width=\linewidth]{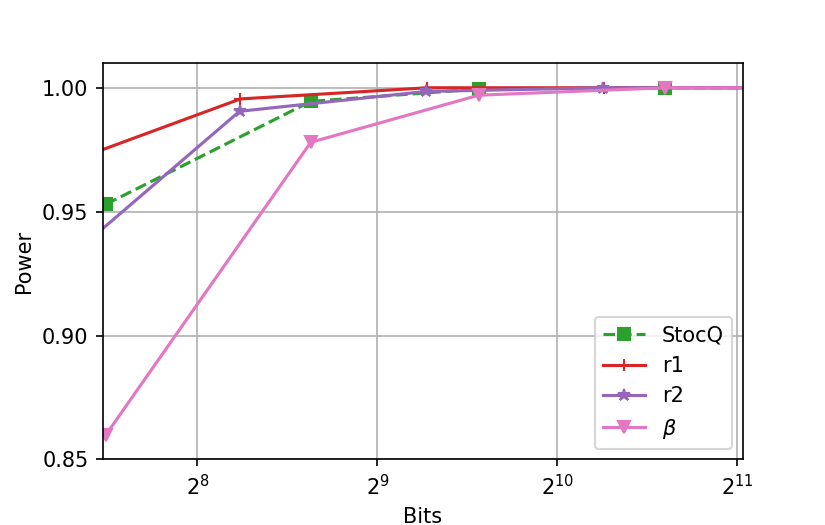}  
\end{subfigure}
\caption{Power of the permutation test with $b=3$.}
\label{fig:mmd_b3}
\vspace{-10pt}
\end{figure}

\section{Proof of Theorem~\ref{thm:main-result}}\label{appendix:proofs}
Given $x,y\in\mathcal{X}\subset\mathbb{R}^d$, we use either stable $Q_{\Sigma\Delta}^{(r)}$ or stable $Q_\beta$ to quantize  their RFFs $z(x)$, $z(y)$ as in \eqref{RFFs}. Then we get quantized sequences 
\[
q_{\Sigma\Delta}^{(r)}(x):=Q_{\Sigma\Delta}^{(r)}(z(x)),\quad  q_{\Sigma\Delta}^{(r)}(y):=Q_{\Sigma\Delta}^{(r)}(z(y)),\quad\text{or}\quad q_\beta(x):=Q_\beta (z(x)),\quad q_\beta(y):=Q_\beta (z(y)),
\]
and expect that both 
\[ 
\widehat{k}^{(r)}_{\Sigma\Delta}(x,y) = \langle \widetilde{V}_{\Sigma\Delta}q_{\Sigma\Delta}^{(r)}(x), \widetilde{V}_{\Sigma\Delta}q_{\Sigma\Delta}^{(r)}(y)\rangle,\quad
\widehat{k}_\beta(x,y) = \langle \widetilde{V}_\beta q_\beta(x), \widetilde{V}_\beta q_\beta(y) \rangle
\]
approximate the ground truth $k(x,y)$ well. 

In the case of $\Sigma\Delta$ quantization, by \eqref{diff-sigma-delta}, one can get
\[
\widetilde{V}_{\Sigma\Delta}q_{\Sigma\Delta}^{(r)}(x) = \widetilde{V}_{\Sigma\Delta}z(x)-\widetilde{V}_{\Sigma\Delta}D^ru_x,\quad 
\widetilde{V}_{\Sigma\Delta}q_{\Sigma\Delta}^{(r)}(y) = \widetilde{V}_{\Sigma\Delta}z(y)-\widetilde{V}_{\Sigma\Delta}D^ru_y.
\]
It follows that
\begin{align*}
\widehat{k}_{\Sigma\Delta}^{(r)}(x,y) &= \langle \widetilde{V}_{\Sigma\Delta}q_{\Sigma\Delta}^{(r)}(x), \widetilde{V}_{\Sigma\Delta}q_{\Sigma\Delta}^{(r)}(y)\rangle =  \langle \widetilde{V}_{\Sigma\Delta}z(x), \widetilde{V}_{\Sigma\Delta}z(y)\rangle-\langle \widetilde{V}_{\Sigma\Delta}z(x), \widetilde{V}_{\Sigma\Delta}D^ru_y\rangle \\
&-\langle \widetilde{V}_{\Sigma\Delta}z(y), \widetilde{V}_{\Sigma\Delta}D^ru_x\rangle+ \langle \widetilde{V}_{\Sigma\Delta}D^ru_x, \widetilde{V}_{\Sigma\Delta}D^ru_y\rangle.
\end{align*}
The triangle inequality implies that
\begin{align}\label{approx-ineq-sigma-delta}
    \bigl|\widehat{k}_{\Sigma\Delta}^{(r)}(x,y)-k(x,y)\bigr| &\leq \underbrace{\bigl| \langle \widetilde{V}_{\Sigma\Delta}z(x), \widetilde{V}_{\Sigma\Delta}z(y)\rangle-k(x,y)\bigr| }_{(\text{I})}
    +\underbrace{\bigl|\langle \widetilde{V}_{\Sigma\Delta}z(x), \widetilde{V}_{\Sigma\Delta}D^ru_y\rangle\bigr|}_{(\text{II})} \\
    &  +\underbrace{\bigl| \langle \widetilde{V}_{\Sigma\Delta}z(y), \widetilde{V}_{\Sigma\Delta}D^ru_x\rangle\bigr|}_{(\text{III})}+\underbrace{\bigl|\langle \widetilde{V}_{\Sigma\Delta}D^ru_x, \widetilde{V}_{\Sigma\Delta}D^ru_y\rangle\bigr|}_{(\text{IV})}. \nonumber
\end{align}
Similarly, for the noise-shaping quantization, one can derive the following inequality based on \eqref{diff-beta},
\begin{align}\label{approx-ineq-beta}
    \bigl|\widehat{k}_\beta^{(r)}(x,y)-k(x,y)\bigr| &\leq \underbrace{\bigl| \langle \widetilde{V}_\beta z(x), \widetilde{V}_\beta z(y)\rangle-k(x,y)\bigr|}_{(\text{I})}
    +\underbrace{\bigl|\langle \widetilde{V}_\beta z(x), \widetilde{V}_\beta Hu_y\rangle\bigr|}_{(\text{II})} \\
    &  +\underbrace{\bigl| \langle \widetilde{V}_\beta z(y), \widetilde{V}_\beta Hu_x\rangle\bigr|}_{(\text{III})}+\underbrace{\bigl|\langle \widetilde{V}_\beta Hu_x, \widetilde{V}_\beta Hu_y\rangle\bigr|}_{(\text{IV})}. \nonumber
\end{align}
In order to control the kernel approximation errors in \eqref{approx-ineq-sigma-delta} and \eqref{approx-ineq-beta}, we need to bound fours terms (I), (II), (III), and (IV) on the right hand side. 

\subsection{Useful Lemmata} 
In this section, we present the following well-known concentration inequalities and relevant lemmata.
\begin{theorem}[Hoeffding's inequality \cite{foucartmathematical}]
Let $X_1,\ldots, X_M$ be a sequence of independent random variables such that $\E X_l=0$ and $|X_l|\leq B_l$ almost surely for all $1\leq l\leq M$. Then for all $t>0$,
\[
\prob\Bigl( \Bigl|\sum_{l=1}^M X_l\Bigr|\geq t\Bigr)\leq 2\exp\Bigl(-\frac{t^2}{2\sum_{l=1}^M B_l^2}\Bigr).
\]
\end{theorem}

\begin{theorem}[Bernstein's inequality \cite{foucartmathematical}]\label{ineq-Bernstein}
Let $X_1,\ldots, X_M$ be independent random variables with zero mean such that $|X_l|\leq K$ almost surely for all $1\leq l\leq M$ and some constant $K>0$. Furthermore assume $\E|X_l|^2\leq \sigma_l^2$ for constants $\sigma_l>0$ for all $1\leq l\leq M$. Then for all $t>0$,
\[
\prob\Bigl( \Bigl|\sum_{l=1}^M X_l\Bigr|\geq t\Bigr)\leq 2\exp\Bigl(-\frac{t^2/2}{\sigma^2+Kt/3}\Bigr)
\]
where $\sigma^2:=\sum_{l=1}^M\sigma_l^2$.
\end{theorem}
Additionally, one can compute the moments of $\cos(\omega_i^\top x+\xi_i)\cos(\omega_j^\top y +\xi_j)$ as follows.
\begin{lemma}\label{identities}
Suppose $x,y\in\mathbb{R}^d$ with RFFs $z(x)$ and $z(y)$ as in \eqref{RFFs}. Let $\widetilde{V}$ be either of the two normalized condensation operators defined in \eqref{def:condensation-normalized}. Then 
\begin{equation}\label{identity1}
	\E(\cos(\omega_j^\top x+\xi_j)\cos(\omega_j^\top y +\xi_j)) = \frac{1}{2}k(x,y),\quad j=1,2,\ldots,m.
\end{equation}
\begin{equation}\label{identity2}
		\E(\cos^2(\omega_i^\top x+\xi_i)\cos^2(\omega_j^\top y +\xi_j))=
\begin{cases}
\frac{1}{4}+\frac{1}{8}k(2x,2y) &\text{if}\; i=j, \\
\frac{1}{4}  &\text{if}\; i\neq j.
\end{cases}
\end{equation}
\begin{equation}\label{identity3}
        \E(\langle \widetilde{V}z(x), \widetilde{V}z(y)\rangle) =k(x,y).
\end{equation}
\end{lemma}
\begin{proof}
(i) Using trigonometric identities, the independence of $\omega_j$ and $\xi_j$ and formula \eqref{Bochner}, we get
\[
\E(\cos(\omega_j^\top x+\xi_j)\cos(\omega_j^\top y +\xi_j)) = \frac{1}{2}\E_{\omega_j\sim\Lambda}\cos(\omega_j^\top(x-y))=\frac{1}{2}\kappa(x-y)=\frac{1}{2}k(x,y).
\]
\noindent (ii) If $i=j$, then
\begin{align*}
		\E(\cos^2(\omega_i^\top x+\xi_i)\cos^2(\omega_i^\top y +\xi_i))&= 
		\frac{1}{4}\E\Bigl((1+\cos(2\omega_i^\top x+2\xi_i))(1+\cos(2\omega_i^\top y+2\xi_i))\Bigr)\nonumber\\
		&=\frac{1}{4}\Bigl(1+\E(\cos(2\omega_i^\top x+2\xi_i)\cos(2\omega_i^\top y +2\xi_i))\Bigr)\\
		&=\frac{1}{4}+\frac{1}{8}k(2x,2y).\nonumber
\end{align*}
Similarly, when $i\neq j$ we have  
\begin{align*}
			\E(\cos^2(\omega_i^\top x+\xi_i)\cos^2(\omega_j^\top y +\xi_j))&=
			\frac{1}{4}+\frac{1}{4}\E\Bigl(\cos(2\omega_i^\top x+2\xi_i)\cos(2\omega_j^\top y +2\xi_j)\Bigr)\nonumber\\
			&=\frac{1}{4}+\frac{1}{4}\E(\cos(2\omega_i^\top x+2\xi_i))\E(\cos(2\omega_j^\top y+2\xi_j))\\
			&=\frac{1}{4}. \nonumber
\end{align*}
(iii) According to \eqref{identity1}, we have $\E( z(x)z(y)^\top)=\frac{1}{2}k(x,y)I_m$ and thus 
\begin{align*}
    \E(\langle \widetilde{V}z(x), \widetilde{V}z(y)\rangle) &= \E(\tr(z(y)^\top \widetilde{V}^\top \widetilde{V}z(x)) )\\
    &=\E(\tr(\widetilde{V}^\top \widetilde{V}z(x)z(y)^\top)) \\
    &=\tr(\widetilde{V}^\top \widetilde{V}\E(z(x)z(y)^\top)) \\
    &= \frac{1}{2}k(x,y)\|\widetilde{V}\|_F^2\\
    &= k(x,y).
\end{align*}
\end{proof}

\begin{lemma}\label{concentration-anchor-point}
Let $x,y\in\mathbb{R}^d$ and $\epsilon>0$. Then
\[
\prob\Bigl(  \bigl| \langle \widetilde{V}_{\Sigma\Delta}z(x), \widetilde{V}_{\Sigma\Delta}z(y)\rangle-k(x,y)\bigr|\geq \epsilon\Bigr) \leq 2\exp\Bigl( -\frac{\epsilon^2 p}{2+k(2x,2y)-2k^2(x,y)+(4\lambda +2)\epsilon/3} \Bigr), 
\]
\[
\prob\Bigl(  \bigl| \langle \widetilde{V}_\beta z(x), \widetilde{V}_\beta z(y)\rangle-k(x,y)\bigr|\geq \epsilon\Bigr) \leq 2\exp\Bigl( -\frac{\epsilon^2 p}{2+k(2x,2y)-2k^2(x,y)+(4\lambda +2)\epsilon/3} \Bigr).
\]
\end{lemma}
\begin{proof}
(i) We first consider the case of $\Sigma\Delta$ scheme, i.e. (I) in \eqref{approx-ineq-sigma-delta}. Note that 
\[
\langle \widetilde{V}_{\Sigma\Delta}z(x), \widetilde{V}_{\Sigma\Delta}z(y)\rangle =\frac{2}{p\|v\|_2^2}\sum_{i=1}^p S_i(x,y) 
\]
where $S_1(x,y),\ldots, S_p(x,y)$ are i.i.d. with 
\[
S_i(x,y):=\sum_{j,k=1}^\lambda v_jv_k\cos(\omega_{(i-1)\lambda+j}^\top x+\xi_{(i-1)\lambda+j})\cos(\omega_{(i-1)\lambda+k}^\top y+\xi_{(i-1)\lambda+k}).
\]
Due to \eqref{identity1}, \eqref{identity2}, and 
\[
\E\Bigl( S_i(x,y)-\frac{k(x,y)\|v\|_2^2}{2}\Bigr) = 0, 
\]
one can get
\begin{align*}
\Var\Bigl(S_i(x,y)-\frac{k(x,y)\|v\|_2^2}{2} \Bigr)&=\Var(S_i(x,y))\\
&=\frac{1}{8}\Bigl(2(k^2(x,y)+1)\|v\|_2^4+(k(2x,2y)-4k^2(x,y))\sum_{i=1}^\lambda v_i^4 \Bigr)\\
&\leq \frac{\|v\|_2^4}{8}\Bigl(2k^2(x,y)+2+k(2x,2y)-4k^2(x,y) \Bigr) \\
&=\frac{\|v\|_2^4}{8}\Bigl(2+k(2x,2y)-2k^2(x,y) \Bigr),
\end{align*}
and
\[
\Bigl| S_i(x,y)- \frac{k(x,y)\|v\|_2^2}{2}\Bigr|\leq |S_i(x,y)|+\frac{\|v\|_2^2}{2}\leq \|v\|_1^2+\frac{\|v\|_2^2}{2}\leq (\lambda+1/2) \|v\|_2^2
\]
for all $1\leq i\leq p$, it follows immediately from Bernstein's inequality that 
\begin{align*}
\prob\Bigl(  \bigl| \langle \widetilde{V}_{\Sigma\Delta}z(x), \widetilde{V}_{\Sigma\Delta}z(y)\rangle-k(x,y)\bigr|\geq \epsilon\Bigr) &= \prob\Bigl(  \Bigl| \sum_{i=1}^p\Bigl(S_i(x,y)-\frac{k(x,y)\|v\|_2^2}{2} \Bigr)\Bigr|\geq \frac{\epsilon p\|v\|_2^2}{2}\Bigr) \\
&\leq 2\exp\Bigl( -\frac{\epsilon^2 p}{2+k(2x,2y)-2k^2(x,y)+(4\lambda +2)\epsilon/3} \Bigr).
\end{align*}
(ii) Since the proof in part (i) works for all vectors $v\in\mathbb{R}^\lambda$ with nonnegative components, a similar result holds for the noise-shaping case by replacing $V_{\Sigma\Delta}$ and $v$ by $V_\beta$ and $v_\beta$ respectively. 
\end{proof}

\begin{lemma}\label{second-third-terms}
Let $x\in\mathbb{R}^d$ and $\epsilon>0$. Then
\[
\prob\Bigl(\frac{1}{p\|v\|_2^2}\|V_{\Sigma\Delta}z(x)\|_1\geq\epsilon \Bigr) \leq 2p\exp\Bigl(-\frac{ \epsilon^2\|v\|_2^2}{1+2\epsilon\|v\|_\infty/3} \Bigr),
\]
\[
\prob\Bigl(\frac{1}{p\|v\|_2^2}\|V_\beta z(x)\|_1\geq\epsilon \Bigr) \leq 2p\exp\Bigl(-\frac{ \epsilon^2\|v_\beta\|_2^2}{1+2\epsilon\|v_\beta\|_\infty/3} \Bigr).
\]
\end{lemma} 
\begin{proof}
(i) In the case of $\Sigma\Delta$ quantization, we note that $V_{\Sigma\Delta}=I_p\otimes v$ and 
\[
\frac{1}{\|v\|_2^2}V_{\Sigma\Delta}z(x) = \frac{(I_p\otimes v) z(x)}{\|v\|_2^2}=\begin{bmatrix}
R_1(x) \\
\vdots \\
R_p(x)
\end{bmatrix}
\]
where $R_i(x):= \frac{1}{\|v\|_2^2}\sum_{j=1}^\lambda v_j z(x)_{(i-1)\lambda+j}=\frac{1}{\|v\|_2^2}\sum_{j=1}^\lambda v_j\cos(\omega_{(i-1)\lambda+j}^\top x+\xi_{(i-1)\lambda+j})$ for $1\leq i\leq p$.

Since $\E(v_j^2\cos^2(\omega_{(i-1)\lambda+j}^\top x+\xi_{(i-1)\lambda+j}))=v_j^2/2$ and $|v_j\cos(\omega_{(i-1)\lambda+j}^\top x+\xi_{(i-1)\lambda+j})|\leq \|v\|_\infty$ holds for all $i$ and $j$, we can apply Theorem~\ref{ineq-Bernstein} to $R_i(x)$ with $K=\|v\|_\infty$, $M=\lambda$, and $\sigma^2=\frac{\|v\|_2^2}{2}$. Specifically, for all $t>0$, we have 
\begin{equation}\label{Berstein}
\prob(|R_i(x)|\geq t) \leq 2\exp\Bigl(-\frac{ t^2\|v\|_2^2}{1+2t\|v\|_\infty/3} \Bigr).
\end{equation}
Since
\[
\frac{1}{p\|v\|_2^2}\|V_{\Sigma\Delta}z(x)\|_1=\frac{1}{p}\|\frac{1}{\|v\|_2^2}V_{\Sigma\Delta}z(x)\|_1=\frac{1}{p}\sum_{i=1}^p |R_i(x)|,
\]
by union bound, we have  
\begin{align*} 
    \prob\Bigl(\frac{1}{p\|v\|_2^2}\|V_{\Sigma\Delta}z(x)\|_1\geq\epsilon \Bigr) &= \prob\Bigl(\sum_{i=1}^p|R_i(x)|\geq \epsilon p\Bigr)\\
    &\leq \prob\Bigl( \bigcup_{i=1}^p\{|R_i(x)|\geq\epsilon \} \Bigr)\\
    &\leq \sum_{i=1}^p \prob(|R_i(x)|\geq\epsilon)\\
    &\leq 2p\exp\Bigl(-\frac{ \epsilon^2\|v\|_2^2}{1+2\epsilon\|v\|_\infty/3} \Bigr)
\end{align*}
where the last inequality is due to \eqref{Berstein}.

\noindent (ii) Substituting $V_{\Sigma\Delta}$ with $V_\beta=I_p\otimes v_\beta$ leads to a verbatim proof for the second inequality. 
\end{proof}

\subsection{Upper bound of (I)}
This section is devoted to deriving an upper bound of the term (I) in \eqref{approx-ineq-sigma-delta}, \eqref{approx-ineq-beta}. Here, we adapt the proof techniques used in \cite{sutherland2015error}.

According to Theorem~\ref{thm:main-result}, $\mathcal{X}$ is a compact subset of $\mathbb{R}^d$ with \emph{diameter} $\ell=\mathrm{diam}(\mathcal{X})>0$. Then $\mathcal{X}^2:=\mathcal{X}\times\mathcal{X}$ is a compact set in $\mathbb{R}^{2d}$ with diameter $\sqrt{2}\ell$. Additionally, the second moment of distribution $\Lambda$, that is, $\sigma_\Lambda^2:=\E_{\omega\sim\Lambda}\|\omega\|_2^2=\tr(\nabla^2 \kappa(0))$ exists where $\nabla^2 \kappa$ is the Hessian matrix of $\kappa$ in \eqref{Bochner}. We will need the following results in order to obtain a uniform bound of term (I) over $\mathcal{X}$, via an $\epsilon$-net argument.

\begin{lemma}[\cite{Cucker}]\label{covering-number}
Let $B_2^d(\eta):=\{x\in\mathbb{R}^d: \|x\|_2\leq \eta \}$. Then the covering number $\mathcal{N}(B_2^d(\eta),\epsilon)$ satisfies 
\[
 \mathcal{N}(B_2^d(\eta),\epsilon) \leq \Bigl(\frac{4\eta}{\epsilon}\Bigr)^d.
\]
\end{lemma}

\begin{lemma}[Jung's Theorem \cite{danzer1963helly}]\label{Jung-theorem}
Let $K\subseteq\mathbb{R}^d$ be compact with $\mathrm{diam}(K)>0$. Then $K$ is contained in a closed ball with radius 
\[
\eta\leq \mathrm{diam}(K)\sqrt{\frac{d}{2(d+1)}}. 
\]
The boundary case of equality is attained by the regular $n$-simplex.
\end{lemma}

We can now prove the following theorem controlling term (I). 
\begin{theorem}\label{uniform-bound-1}
Let $\epsilon, \eta_1>0$. Then
\begin{align*}
    &\prob\Bigl( \sup_{x,y\in\mathcal{X}}|\langle \widetilde{V}_{\Sigma\Delta}z(x), \widetilde{V}_{\Sigma\Delta}z(y) \rangle-k(x,y)|<\epsilon \Bigr)\\
    &\geq 1-32\sigma_\Lambda^2\Bigl(\frac{\eta_1\lambda}{\epsilon}\Bigr)^2-2\Bigl(\frac{4\ell}{\eta_1}\Bigr)^{2d}\exp\Bigl( -\frac{\epsilon^2 p}{8+4k(2x,2y)-8k^2(x,y)+(8\lambda +4)\epsilon/3} \Bigr),
\end{align*}
and 
\begin{align*}
    &\prob\Bigl( \sup_{x,y\in\mathcal{X}}|\langle \widetilde{V}_\beta z(x), \widetilde{V}_\beta z(y) \rangle-k(x,y)|<\epsilon \Bigr)\\
    &\geq 1-32\sigma_\Lambda^2\Bigl(\frac{\eta_1\lambda}{\epsilon}\Bigr)^2-2\Bigl(\frac{4\ell}{\eta_1}\Bigr)^{2d}\exp\Bigl( -\frac{\epsilon^2 p}{8+4k(2x,2y)-8k^2(x,y)+(8\lambda +4)\epsilon/3} \Bigr).
\end{align*}
\end{theorem}
\begin{proof}
Indeed, the following proof techniques are independent of the choice of row vector $v$ in $\widetilde{V}_{\Sigma\Delta}$. So we only prove the case related to $\widetilde{V}_{\Sigma\Delta}$ and everything works for $\widetilde{V}_\beta$ by replacing $v$ with $v_\beta$. Let 
\[
s(x,y) := \langle \widetilde{V}_{\Sigma\Delta}z(x), \widetilde{V}_{\Sigma\Delta}z(y)\rangle,\quad f(x,y):=s(x,y)-k(x,y).
\]
Recall that $\E(s(x,y))=k(x,y)$ and 
\[
s(x,y)=\frac{2}{p\|v\|_2^2}\sum_{i=1}^p S_i(x,y)
\]
where $S_1(x,y),\ldots, S_p(x,y)$ are i.i.d. with 
\[
S_i(x,y)=\sum_{j,k=1}^\lambda v_jv_k\cos(\omega_{(i-1)\lambda+j}^\top x+\xi_{(i-1)\lambda+j})\cos(\omega_{(i-1)\lambda+k}^\top y+\xi_{(i-1)\lambda+k}).
\]
According to Lemma \ref{Jung-theorem}, $\mathcal{X}^2\subseteq\mathbb{R}^{2d}$ is enclosed in a closed ball with radius $\ell\sqrt{\frac{2d}{2d+1}}$. By Lemma \ref{covering-number}, one can cover $\mathcal{X}^2$ using an $\eta_1$-net with at most $\Bigl(\frac{4\ell}{\eta_1}\sqrt{\frac{2d}{2d+1}}\Bigr)^{2d}\leq T_1:= \bigl(\frac{4\ell}{\eta_1}\big)^{2d}$ balls of radius $\eta_1$. Let $c_i=(x_i, y_i)$ denote their centers for $1\leq i\leq T_1$.

For $1\leq l\leq d$ we have 
\begin{align*}
\Bigl|\frac{\partial s}{\partial x_l}(x,y)\Big| &=\frac{2}{p\|v\|_2^2}\Big|\sum_{i=1}^p \frac{\partial S_i}{\partial x_l}(x,y) \Big|\\
&\leq \frac{2}{p\|v\|_2^2}\sum_{i=1}^p \Bigl| \sum_{j,k=1}^\lambda v_jv_k\sin(\omega_{(i-1)\lambda+j}^\top x+\xi_{(i-1)\lambda+j})\cos(\omega_{(i-1)\lambda+k}^\top y+\xi_{(i-1)\lambda+k})\omega_{(i-1)\lambda+j,l} \Bigr|  \\
&\leq \frac{2}{p\|v\|_2^2}\sum_{i=1}^p \sum_{j,k=1}^\lambda v_jv_k|\omega_{(i-1)\lambda+j,l} |.
\end{align*}
Then 
\[
\E\Big(  \frac{2}{p\|v\|_2^2}\sum_{i=1}^p \sum_{j,k=1}^\lambda v_jv_k|\omega_{(i-1)\lambda+j,l} |\Big)\leq  \frac{2}{p\|v\|_2^2}\sum_{i=1}^p \sum_{j,k=1}^\lambda v_jv_k \E(\|\omega\|_\infty)=
\frac{2\|v\|_1^2}{\|v\|_2^2}E(\|\omega\|_\infty)<\infty.
\]
Since $\Bigl|\frac{\partial s}{\partial x_l}(x,y)\Big| $ is dominated by an integrable function, one can interchange expectations and partial derivatives. In particular, 
\[
\frac{\partial}{\partial x_l}\Big( \E s(x,y)\Big)=\E\Big( \frac{\partial s}{\partial x_l}(x,y)\Big),
\]
and similarly 
\[
\frac{\partial}{\partial y_l}\Big( \E s(x,y)\Big)=\E\Big( \frac{\partial s}{\partial y_l}(x,y)\Big).
\]
It follows that 
\begin{equation}\label{exp-grad-swap}
E\nabla s(x,y)= \nabla\E s(x,y) = \nabla k(x,y).
\end{equation}
Let Lipschitz constant $L_f=\|\nabla f(x^*, y^*)\|_2$ with $(x^*,y^*)=\argmax_{(x,y)\in\mathcal{X}^2}\|\nabla f(x,y) \|_2$. Applying law of total expectation and \eqref{exp-grad-swap} gives
\begin{align}
    \E(L_f^2)&= \E(\|\nabla s(x^*,y^*)-\nabla k(x^*,y^*)\|_2^2) \nonumber\\
    &= \E\big(\E(\|\nabla s(x^*,y^*)-\nabla k(x^*,y^*)\|_2^2 \,|\, x^*,y^*) \big)\nonumber\\
    &= \E\Big( \E(\|\nabla s(x^*,y^*)\|_2^2  \,|\, x^*,y^*)  + \|\nabla k(x^*,y^*)\|_2^2-2\E(\langle \nabla s(x^*, y^*), \nabla k(x^*, y^*)\rangle \,|\, x^*,y^*) \Big) \nonumber\\
    &=\E(\|\nabla s(x^*,y^*)\|_2^2) + \E\Big( \|\nabla k(x^*,y^*)\|_2^2-2\langle \E(\nabla s(x^*, y^*) \,|\, x^*,y^*), \nabla k(x^*, y^*)\rangle\Big)\nonumber\\
    &=\E(\|\nabla s(x^*,y^*)\|_2^2) + \E\Big( \|\nabla k(x^*,y^*)\|_2^2-2\langle \nabla k(x^*,y^*), \nabla k(x^*, y^*)\rangle\Big) \nonumber\\
    &=\E(\|\nabla s(x^*,y^*)\|_2^2) - \E( \|\nabla k(x^*,y^*)\|_2^2) \nonumber\\
    &\leq \E(\|\nabla s(x^*,y^*)\|_2^2)\nonumber \\
    &= \E(\|\nabla_x s(x^*,y^*)\|_2^2) + \E(\|\nabla_y s(x^*,y^*)\|_2^2). \label{exp-Lipschitz}
\end{align}

Note that 
\begin{align*}
&\|\nabla_x s(x^*,y^*)\|_2 \\
&\leq \frac{2}{p\|v\|_2^2}\sum_{i=1}^p\| \nabla_x S_i(x^*,y^*)\|_2 \\
&=\frac{2}{p\|v\|_2^2}\sum_{i=1}^p \Bigl\| \sum_{j,k=1}^\lambda v_jv_k\sin(\omega_{(i-1)\lambda+j}^\top x^*+\xi_{(i-1)\lambda+j})\cos(\omega_{(i-1)\lambda+k}^\top y^*+\xi_{(i-1)\lambda+k})\omega_{(i-1)\lambda+j} \Bigr\|_2\\
&= \frac{2}{p\|v\|_2^2}\sum_{i=1}^p \Bigl\|  \sum_{k=1}^\lambda v_k \cos(\omega_{(i-1)\lambda+k}^\top y^*+\xi_{(i-1)\lambda+k}) \cdot \sum_{j=1}^\lambda v_j \sin(\omega_{(i-1)\lambda+j}^\top x^*+\xi_{(i-1)\lambda+j}) \omega_{(i-1)\lambda+j}\Big\|_2      \\ 
&= \frac{2}{p\|v\|_2^2}\sum_{i=1}^p\Big|\sum_{k=1}^\lambda v_k \cos(\omega_{(i-1)\lambda+k}^\top y^*+\xi_{(i-1)\lambda+k})\Big| \cdot\Bigl\|\sum_{j=1}^\lambda v_j \sin(\omega_{(i-1)\lambda+j}^\top x^*+\xi_{(i-1)\lambda+j}) \omega_{(i-1)\lambda+j} \Big\|_2 \\
&\leq \frac{2\|v\|_1}{p\|v\|_2^2} \sum_{i=1}^p\sum_{j=1}^\lambda v_j\|\omega_{(i-1)\lambda+j}\|_2.
\end{align*}

By Cauchy–Schwarz inequality and the fact $\|v\|_1\leq\sqrt{\lambda}\|v\|_2$, we have 
\begin{align*}
    \|\nabla_x s(x^*,y^*)\|_2^2 &\leq \frac{4\|v\|_1^2}{p^2\|v\|_2^4} \Bigl(\sum_{i=1}^p\sum_{j=1}^\lambda v_j\|\omega_{(i-1)\lambda+j}\|_2\Bigr)^2 \\
    &\leq  \frac{4\|v\|_1^2}{p\|v\|_2^2} \sum_{i=1}^p\sum_{j=1}^\lambda \|\omega_{(i-1)\lambda+j}\|_2^2 \\
    &\leq\frac{4\lambda}{p} \sum_{i=1}^p\sum_{j=1}^\lambda \|\omega_{(i-1)\lambda+j}\|_2^2 .
\end{align*}
Then 
\[
\E(\|\nabla_x s(x^*,y^*)\|_2^2) \leq 4\lambda^2 \E(\|\omega\|_2^2)=4\lambda^2\sigma_\Lambda^2
\]
and similarly 
\[
\E(\|\nabla_y s(x^*,y^*)\|_2^2) \leq 4\lambda^2\sigma_\Lambda^2.
\]
Plugging above results into \eqref{exp-Lipschitz} shows
\[
\E(L_f^2)\leq 8\lambda^2\sigma_\Lambda^2.
\]
Let $\epsilon>0$. Markov's inequality implies
\begin{equation}\label{Lipschitz-bound}
    \prob\Bigl(L_f\geq \frac{\epsilon}{2\eta_1} \Bigr) \leq \frac{4\eta_1^2\E(L_f^2)}{\epsilon^2}\leq 32\sigma_\Lambda^2\Bigl(\frac{\eta_1\lambda}{\epsilon}\Bigr)^2.
\end{equation}
By union bound and Lemma \ref{concentration-anchor-point}, we get
\begin{align}\label{anchor-bound}
\prob\Bigl( \bigcup_{i=1}^{T_1} \{|f(c_i)|\geq \epsilon/2\} \Bigr) &\leq T_1\prob\Bigl( |f(c_i)|\geq \epsilon/2\Bigr) \nonumber \\ 
&\leq 2\Bigl(\frac{4\ell}{\eta_1}\Bigr)^{2d}\exp\Bigl( -\frac{\epsilon^2 p}{8+4k(2x,2y)-8k^2(x,y)+(8\lambda +4)\epsilon/3} \Bigr).
\end{align} 

If $|f(c_i)|<\epsilon/2$ for all $i$ and $L_f<\epsilon/2\eta_1$, then $|f(x,y)|<\epsilon$ for all $(x,y)\in\mathcal{X}^2$. It follows immediately from \eqref{Lipschitz-bound} and \eqref{anchor-bound} that
\begin{align*}
&\prob\Bigl( \sup_{x,y\in\mathcal{X}}|f(x,y)|<\epsilon \Bigr)\\
&\geq 1-32\sigma_\Lambda^2\Bigl(\frac{\eta_1\lambda}{\epsilon}\Bigr)^2-2\Bigl(\frac{4\ell}{\eta_1}\Bigr)^{2d}\exp\Bigl( -\frac{\epsilon^2 p}{8+4k(2x,2y)-8k^2(x,y)+(8\lambda +4)\epsilon/3} \Bigr).
\end{align*}
\end{proof}

\subsection{Upper bound of (II) \& (III)}
By symmetry it suffices to bound (II) in \eqref{approx-ineq-sigma-delta}, \eqref{approx-ineq-beta}, and the same upper bound holds for (III).
\begin{theorem}\label{uniform-bound-2}
Let $\epsilon,\eta_2>0$. Then we have 
\begin{align*}
&\prob\Bigl( \sup_{x,y\in\mathcal{X}}\bigl|\langle \widetilde{V}_{\Sigma\Delta}z(x), \widetilde{V}_{\Sigma\Delta}D^ru_y\rangle\bigr|<\epsilon \Bigr) \\ &\geq 1-\sigma_\Lambda^2\Bigl(\frac{c(K,r)2^{r+2}\eta_2}{\epsilon}\Bigr)^2-2p\Bigl(\frac{2\sqrt{2}\ell}{\eta_2}\Bigr)^{d}\exp\Bigl(-\frac{ \epsilon^2\|v\|_2^2}{c(K,r)^22^{2r+4}+c(K,r)2^{r+3}\epsilon\|v\|_\infty/3} \Bigr). 
\end{align*} 
and 
\begin{align*}
&\prob\Bigl( \sup_{x,y\in\mathcal{X}}\bigl|\langle \widetilde{V}_\beta z(x), \widetilde{V}_\beta Hu_y\rangle\bigr|<\epsilon \Bigr) \\ &\geq 1-\sigma_\Lambda^2\Bigl(\frac{4\eta_2c(K,\beta)}{\epsilon\beta^\lambda}\Bigr)^2-2p\Bigl(\frac{2\sqrt{2}\ell}{\eta_2}\Bigr)^{d}\exp\Bigl(-\frac{ \epsilon^2\|v_\beta\|_2^2\beta^{2\lambda}}{16c(K,\beta)^2+8\beta^\lambda c(K,\beta)\epsilon\|v_\beta\|_\infty/3} \Bigr),
\end{align*} 
where $c(K,r)$ and $c(K,\beta)$ are upper bounds of the $\ell_\infty$ norm of state vectors in Proposition~\ref{prop:stability-sigma-delta} and Proposition~\ref{prop:stability-beta} respectively.
\end{theorem}
\begin{proof} 
(i) We first prove the case associated with $\widetilde{V}_{\Sigma\Delta}$. Define $g(x):=\frac{1}{p\|v\|_2^2}\|V_{\Sigma\Delta}z(x)\|_1$ for $x\in\mathbb{R}^d$. Since $\mathrm{diam}(\mathcal{X})=\ell$, by Lemma~\ref{covering-number} and Lemma~\ref{Jung-theorem}, one can cover $\mathcal{X}$ using an $\eta_2$-net with at most $T_2:=\big( \frac{2\sqrt{2}\ell}{\eta_2}\big)^d$ balls with radius $\eta_2$. Let $x_i$ denote their centers for $1\leq i\leq T_2$ and let Lipschitz constant $L_g=\|\nabla g(x^*)\|_2$ with $x^*=\argmax_{x\in\mathcal{X}}\|\nabla g(x) \|_2$.

Note that 
\begin{align*}
g(x) = \frac{1}{p\|v\|_2^2}\sum_{i=1}^p|g_i(x)|
\end{align*}
where $g_i(x)=\sum_{j=1}^\lambda v_j\cos(\omega_{(i-1)\lambda+j}^\top x+\xi_{(i-1)\lambda+j})$. 

It follows that
\begin{align*}
\|\nabla g(x^*) \|_2 &= \frac{1}{p\|v\|_2^2} \Bigl\| \sum_{i=1}^p\sum_{j=1}^\lambda \sign(g_i(x^*))v_j\sin(\omega_{(i-1)\lambda+j}^\top x^*+\xi_{(i-1)\lambda+j})\omega_{(i-1)\lambda+j}\Bigr\|_2 \\
&\leq \frac{1}{p\|v\|_2^2}  \sum_{i=1}^p\sum_{j=1}^\lambda v_j \|\omega_{(i-1)\lambda+j} \|_2.
\end{align*}
Applying Cauchy-Schwarz inequality gives
\[
\|\nabla g(x^*) \|_2^2\leq \frac{1}{p^2\|v\|_2^4}\Bigl(\sum_{i=1}^p\sum_{j=1}^\lambda v_j^2\Bigr)\Bigl( \sum_{i=1}^p\sum_{j=1}^\lambda \|\omega_{(i-1)\lambda+j}\|_2^2 \Bigr)=
 \frac{1}{p\|v\|_2^2} \sum_{i=1}^p\sum_{j=1}^\lambda \|\omega_{(i-1)\lambda+j}\|_2^2.
\]
Taking expectation on both sides leads to
\[
\E(L_g^2)=\E(\|\nabla g(x^*)\|_2^2)\leq \frac{\lambda}{\|v\|_2^2}\E(\|\omega\|_2^2)\leq \E(\|\omega\|_2^2)=\sigma_\Lambda^2.
\] 
Let $\epsilon>0$. Markov's inequality implies
\begin{equation}\label{Lipschitz-bound-2}
    \prob\Bigl(L_g\geq \frac{\epsilon}{2\eta_2} \Bigr) \leq \frac{4\eta_2^2\E(L_g^2)}{\epsilon^2}\leq \sigma_\Lambda^2\Bigl(\frac{2\eta_2}{\epsilon}\Bigr)^2.
\end{equation}
By union bound and Lemma \ref{second-third-terms}, we get
\begin{align}\label{anchor-bound-2}
\prob\Bigl( \bigcup_{i=1}^{T_2} \{|g(x_i)|\geq \epsilon/2\} \Bigr) &\leq T_2\prob\Bigl( |g(x_i)|\geq \epsilon/2\Bigr) \nonumber \\ 
&\leq 2p\Bigl(\frac{2\sqrt{2}\ell}{\eta_2}\Bigr)^{d}\exp\Bigl(-\frac{ \epsilon^2\|v\|_2^2}{4+4\epsilon\|v\|_\infty/3} \Bigr).
\end{align} 

If $|g(x_i)|<\epsilon/2$ for all $i$ and $L_g<\epsilon/2\eta_2$, then $|g(x)|<\epsilon$ for all $x\in\mathcal{X}^2$. It follows immediately from \eqref{Lipschitz-bound-2} and \eqref{anchor-bound-2} that
\begin{align}\label{anchor-bound-3}
    \prob\Bigl( \sup_{x\in\mathcal{X}}\|V_{\Sigma\Delta}z(x)\|_1<p\epsilon\|v\|_2^2 \Bigr)&=\prob\Bigl( \sup_{x\in\mathcal{X}}|g(x)|<\epsilon \Bigr)\\
    &\geq 1-\sigma_\Lambda^2\Bigl(\frac{2\eta_2}{\epsilon}\Bigr)^2-2p\Bigl(\frac{2\sqrt{2}\ell}{\eta_2}\Bigr)^{d}\exp\Bigl(-\frac{ \epsilon^2\|v\|_2^2}{4+4\epsilon\|v\|_\infty/3} \Bigr).\nonumber 
\end{align}

Because $\|V_{\Sigma\Delta}D^r\|_\infty=2^r$ and $\|u_y\|_\infty\leq c(K, r):=c(K,1,r)$ in Proposition~\ref{prop:stability-sigma-delta}, we have 
\begin{align*}
\bigl|\langle \widetilde{V}_{\Sigma\Delta}z(x), \widetilde{V}_{\Sigma\Delta}D^ru_y\rangle\bigr| &= \frac{2}{p\|v\|_2^2}  \bigl|\langle V_{\Sigma\Delta}z(x), V_{\Sigma\Delta}D^ru_y\rangle\bigr| \\
&\leq \frac{2}{p\|v\|_2^2}\|V_{\Sigma\Delta}z(x)\|_1 \|V_{\Sigma\Delta}D^ru_y\|_\infty\\
&\leq \frac{2}{p\|v\|_2^2}\|V_{\Sigma\Delta}z(x)\|_1 \|V_{\Sigma\Delta}D^r\|_\infty \|u_y\|_\infty\\
&\leq 2^{r+1}c(K,r) g(x).
\end{align*}
Therefore, one can get
\begin{align*}
&\prob\Bigl( \sup_{x,y\in\mathcal{X}}\bigl|\langle \widetilde{V}_{\Sigma\Delta}z(x), \widetilde{V}_{\Sigma\Delta}D^ru_y\rangle\bigr|<\epsilon \Bigr) \\ &\geq 1-\sigma_\Lambda^2\Bigl(\frac{c(K,r)2^{r+2}\eta_2}{\epsilon}\Bigr)^2-2p\Bigl(\frac{2\sqrt{2}\ell}{\eta_2}\Bigr)^{d}\exp\Bigl(-\frac{ \epsilon^2\|v\|_2^2}{c(K,r)^22^{2r+4}+c(K,r)2^{r+3}\epsilon\|v\|_\infty/3} \Bigr). 
\end{align*}
(ii) By repeating the statements before \eqref{anchor-bound-3} with $V_{\Sigma\Delta}$ replaced with $V_\beta$, one can get
\begin{equation}\label{anchor-bound-4}
    \prob\Bigl( \sup_{x\in\mathcal{X}}\|V_\beta z(x)\|_1<p\epsilon\|v_\beta\|_2^2 \Bigr)
    \geq 1-\sigma_\Lambda^2\Bigl(\frac{2\eta_2}{\epsilon}\Bigr)^2-2p\Bigl(\frac{2\sqrt{2}\ell}{\eta_2}\Bigr)^{d}\exp\Bigl(-\frac{ \epsilon^2\|v_\beta\|_2^2}{4+4\epsilon\|v_\beta\|_\infty/3} \Bigr).
\end{equation}
Due to $\|V_\beta H\|_\infty = \beta^{-\lambda}$ and $\|u_y\|_\infty\leq c(K,\beta)$ in Proposition~\ref{prop:stability-beta}, we get 
\begin{align}\label{anchor-bound-5}
\bigl|\langle \widetilde{V}_\beta z(x), \widetilde{V}_\beta Hu_y\rangle\bigr| &= \frac{2}{p\|v_\beta\|_2^2}  \bigl|\langle V_\beta z(x), V_\beta Hu_y\rangle\bigr| \\
&\leq \frac{2}{p\|v_\beta\|_2^2}\|V_\beta z(x)\|_1 \|V_\beta Hu_y\|_\infty\nonumber \\
&\leq \frac{2}{p\|v_\beta\|_2^2}\|V_\beta z(x)\|_1 \|V_\beta H\|_\infty \|u_y\|_\infty \nonumber\\
&\leq  \frac{2\beta^{-\lambda}c(K,\beta)}{p\|v_\beta\|_2^2} \|V_\beta z(x)\|_1 .\nonumber
\end{align}
It follows from \eqref{anchor-bound-4}, \eqref{anchor-bound-5} that 
\begin{align*}
&\prob\Bigl( \sup_{x,y\in\mathcal{X}}\bigl|\langle \widetilde{V}_\beta z(x), \widetilde{V}_\beta Hu_y\rangle\bigr|<\epsilon \Bigr) \\ &\geq 1-\sigma_\Lambda^2\Bigl(\frac{4\eta_2c(K,\beta)}{\epsilon\beta^\lambda}\Bigr)^2-2p\Bigl(\frac{2\sqrt{2}\ell}{\eta_2}\Bigr)^{d}\exp\Bigl(-\frac{ \epsilon^2\|v_\beta\|_2^2\beta^{2\lambda}}{16c(K,\beta)^2+8\beta^\lambda c(K,\beta)\epsilon\|v_\beta\|_\infty/3} \Bigr). 
\end{align*}
\end{proof}

\subsection{Upper Bound of (IV)}
\begin{theorem}\label{uniform-bound-3}
Let $r\in\mathbb{N}^+$ and $\beta\in(1,2)$.
\begin{enumerate}
    \item If $u_x$, $u_y$ are state vectors of the $\Sigma\Delta$ quantizer $Q_{\Sigma\Delta}^{(r)}$, then
    \[
\bigl|\langle \widetilde{V}_{\Sigma\Delta}D^ru_x, \widetilde{V}_{\Sigma\Delta}D^ru_y\rangle\bigr|\leq\frac{c(K,r)^2c(r)}{\lambda^{2r-1}},
    \]
    where $c(K,r)$ is the upper bound of the $\ell_\infty$ norm of state vectors in Proposition~\ref{prop:stability-sigma-delta} and $c(r)>0$ is a constant related to $r$.
    \item If $u_x$, $u_y$ are state vectors of the noise-shaping quantizer $Q_\beta$, then
    \[
    \bigl|\langle \widetilde{V}_\beta Hu_x, \widetilde{V}_\beta Hu_y\rangle\bigr| \leq \frac{2c(K,\beta)^2}{\beta^{2\lambda -2}},
    \]
    where $c(K,\beta)$ is the upper bound of the $\ell_\infty$ norm of state vectors in Proposition~\ref{prop:stability-beta}.
\end{enumerate}
\end{theorem}
\begin{proof}
(i)
Cauchy-Schwarz inequality implies 
\[
\bigl|\langle \widetilde{V}_{\Sigma\Delta}D^ru_x, \widetilde{V}_{\Sigma\Delta}D^ru_y\rangle\bigr|\leq \frac{2}{p\|v\|_2^2}\|V_{\Sigma\Delta}D^ru_x\|_2 \|V_{\Sigma\Delta}D^r u_y\|_2. 
\]
One can easily verify that $V_{\Sigma\Delta}D^r$ is a sparse matrix such that each row has at most $r+1$ nonzero entries $\{w_0,w_1,\ldots,w_r \}$ of the following form
\[
w_k=(-1)^{r+k}\binom{r}{k}.
\]
Since $\max\{\|u_x\|_\infty,\|u_y\|_\infty\}\leq c(K,r)$ as indicated by Proposition~\ref{prop:stability-sigma-delta}, we have $\|V_{\Sigma\Delta}D^r u_x\|_2\leq c(K, r)c(r)\sqrt{p}$ and $\|V_{\Sigma\Delta}D^r u_y\|_2\leq c(K,r)c(r)\sqrt{p}$. So above inequality becomes
\[
\bigl|\langle \widetilde{V}_{\Sigma\Delta}D^r u_x, \widetilde{V}_{\Sigma\Delta}D^r u_y\rangle\bigr|\leq \frac{2}{p\|v\|_2^2}\|V_{\Sigma\Delta}D^r u_x\|_2 \|V_{\Sigma\Delta}D^r u_y\|_2
\leq \frac{2c(r)^2 c(K,r)^2}{\|v\|_2^2} \leq \frac{c(K,r)^2c^\prime(r)}{\lambda^{2r-1}}
\]
where the last inequality is due to $\|v\|_2^2\geq \lambda^{2r-1}r^{-2r}$.

\noindent (ii) In the case of noise-shaping quantization, similarly, we have 
\[
\bigl|\langle \widetilde{V}_\beta Hu_x, \widetilde{V}_\beta Hu_y\rangle\bigr|\leq \frac{2}{p\|v_\beta\|_2^2}\|V_\beta Hu_x\|_2 \|V_\beta H u_y\|_2. 
\]
Note that $V_\beta H=(I_p\otimes v_\beta)(I_p\otimes H_\beta) =I_p\otimes (v_\beta H_\beta)$ with $v_\beta H_\beta=(0,0,\ldots,0,\beta^{-\lambda})\in\mathbb{R}^{1\times\lambda}$, and $\max\{\|u_x\|_\infty,\|u_y\|_\infty\}\leq c(K,\beta)$ by Proposition~\ref{prop:stability-beta}. 
It follows that $\|V_\beta Hu_x\|_2\leq \beta^{-\lambda}\sqrt{p}\|u_x\|_\infty$ and $\|V_\beta Hu_y\|_2\leq \beta^{-\lambda}\sqrt{p}\|u_y\|_\infty$. Then one can get
\[
\bigl|\langle \widetilde{V}_\beta Hu_x, \widetilde{V}_\beta Hu_y\rangle\bigr|\leq \frac{2}{p\|v_\beta\|_2^2}\|V_\beta Hu_x\|_2 \|V_\beta H u_y\|_2 \leq \frac{2\beta^{-2\lambda}c(K,\beta)^2}{\|v_\beta\|_2^2} \leq \frac{2c(K,\beta)^2}{\beta^{2\lambda -2}}
\]
where the last inequality comes from $\|v_\beta \|_2\geq \beta^{-1}$.
\end{proof}

\subsection{Proof of Theorem~\ref{thm:main-result}}
\begin{proof}
Recall that the kernel approximation errors in \eqref{approx-ineq-sigma-delta} and \eqref{approx-ineq-beta} can be bounded by four terms (I), (II), (III), (IV). 

\noindent (i) For the $\Sigma\Delta$ scheme, in Theorem~\ref{uniform-bound-1}, one can choose $\epsilon=O(\sqrt{p^{-1}\log p})$, $\lambda=O(\sqrt{p\log^{-1} p})$ and $\eta_1=O(p^{-2-\alpha})$ with $\alpha>0$. Moreover, since $\|v\|_2^2\geq \lambda^{2r-1}r^{-2r}$ and $\|v\|_\infty=O(\lambda^{r-1})$ (see Lemma 4.6 in \cite{huynh2020fast}), in Theorem \ref{uniform-bound-2}, we can choose $\epsilon=O(c(K,r)\lambda^{-r+1}\log^{1/2}p)$ and $\eta_2=O(\lambda^{-r-1}\log^{1/2}p)$. Then \eqref{uniform-bound-sigma-delta} follows immediately by combining above results with part (1) in Theorem \ref{uniform-bound-3}. 

\noindent (ii) As for the noise-shaping scheme, in Theorem~\ref{uniform-bound-1}, we choose the same parameters as in part (i): $\epsilon=O(\sqrt{p^{-1}\log p})$, $\lambda=O(\sqrt{p\log^{-1} p})$ and $\eta_1=O(p^{-2-\alpha})$ with $\alpha>0$. Nevertheless, according to $\|v_\beta\|_2^2\geq \beta^{-2}$ and $\|v_\beta\|_\infty=\beta^{-1}$, we set different values $\epsilon=O(c(K,\beta)\beta^{-\lambda+1}\sqrt{p})$ and $\eta_2=O(p^{-1})$ in Theorem~\ref{uniform-bound-2}. Therefore, \eqref{uniform-bound-beta} holds by applying above results and part (2) in Theorem~\ref{uniform-bound-3}.
\end{proof}

\section{Proof of theorem~\ref{specBoundtheorem}}
\label{specBoundtheoremProof}
The architecture for the proof of theorem \ref{specBoundtheorem} closely follows the methods used in \cite{zhang2019low}. We start with some useful lemmata that aid in proving theorem \ref{specBoundtheorem}.

Given a $b$-bit alphabet as in \eqref{alphabet} with $b=\log_2(2K)$, we consider the following first-order $\Sigma\Delta$ quantization scheme for a random Fourier feature vector $z(x) \in [-1,1]^m$ corresponding to a data point $x \in \mathbb{R}^d$, where, the state variable $(u_x)_0$ is initialized as a random number, i.e.
\[
    \begin{split}
        (u_x)_0 & \sim U\left[-\frac{1}{2^b-1},\frac{1}{2^b-1}\right]\\
        q_{k+1} & = Q_{MSQ}((z(x))_{k+1} + (u_x)_{k})\\
        (u_x)_{k+1} & = (u_x)_{k} + (z(x)))_{k+1} - q_{k+1}
    \end{split}
\]
The corresponding recurrence equation can written as
\[
\Tilde{V}Q(z(x)) = \Tilde{V}z(x) - \Tilde{V}Du_x + \Tilde{V}(u_0^x,0,\ldots,0)^\top.
\]

\begin{lemma}
\label{stateVecUniform}
Given the following first order Sigma-Delta quantization scheme with a $b$-bit alphabet as in \eqref{alphabet}, for a vector $z \in \mathbb{R}^m$ with $z \in [-1,1]^m$, 
\begin{equation*}
\begin{split}
        u_0 & \sim U\left[-\frac{1}{2^b-1},\frac{1}{2^b-1}\right]\\
        q_{k+1} & = Q_{MSQ}(z_{k+1} + u_{k})\\
        u_{k+1} & = u_{k} + z_{k+1} - q_{k+1},
    \end{split}
\end{equation*}
for each $k=0,1,\cdots m-1$, we have  $u_{k} \sim U\left[-\frac{1}{2^b-1},\frac{1}{2^b-1}\right]$.
\end{lemma}

\begin{proof}
Let the inductive hypothesis be $u_k  \sim U\left[-\frac{1}{2^b-1},\frac{1}{2^b-1}\right]$. Note that this is true by definition for $u_0$.

\textbf{Case:} $\frac{j}{2^b-1} \leq z_{k+1} \leq \frac{j+1}{2^b-1}$ where $j \in \{ 1, 3,\cdots  2^b-3\}$.

$u_k \sim U\left[-\frac{1}{2^b-1},\frac{1}{2^b-1}\right]$ implies that $z_{k+1}+u_k \sim U\left[-\frac{1}{2^b-1} + z_{k+1},\frac{1}{2^b-1} + z_{k+1}\right]$. Since by assumption, $\frac{j}{2^b-1} \leq z_{k+1} \leq \frac{j+1}{2^b-1}$ we see that $z_{k+1}+u_k \in [\frac{j-1}{2^b-1},\frac{j+2}{2^b-1}]$ and thus
\begin{equation*}
    Q_{MSQ}(z_{k+1}+u_k) = 
    \begin{cases}
    \frac{j}{2^b-1} & \text{ if }\ \frac{j-1}{2^b-1} \leq z_{k+1}+u_k \leq \frac{j+1}{2^b-1}, \\
    \frac{j+2}{2^b-1} & \text{ if }\ \frac{j+1}{2^b-1} \leq z_{k+1}+u_k \leq \frac{j+2}{2^b-1},
    \end{cases}
\end{equation*}
which in turn implies that
\begin{equation*}
    u_{k+1} = 
    \begin{cases}
    z_{k+1} + u_k - \frac{j}{2^b-1} & \text{ if } \ \frac{j-1}{2^b-1} \leq z_{k+1}+u_k \leq \frac{j+1}{2^b-1}, \\
    z_{k+1} + u_k - \frac{j+2}{2^b-1} & \text{ if }\  \frac{j+1}{2^b-1} \leq z_{k+1}+u_k \leq \frac{j+2}{2^b-1}.
    \end{cases}
\end{equation*}
Now we can compute the CDF of $u_{k+1}$ (conditioned on $z$) as follows
\begin{equation*}
    \begin{split}
    \mathbb{P}(u_{k+1} \leq \alpha \: | \: z) & = \mathbb{P}\Bigl(z_{k+1} + u_k - \frac{j}{2^b-1} \leq \alpha \: , \: q_{k} = \frac{j}{2^b-1}  \: \Big| \: z \Bigr) \\
    & + \mathbb{P}\Bigl(z_{k+1} + u_k - \frac{j+2}{2^b-1} \leq \alpha \: , \: q_{k} = \frac{j+2}{2^b-1}  \: \Big| \:  z \Bigr)\\
    &  = \mathbb{P}\Bigl( \frac{j-1}{2^b-1} - z_{k+1} \leq u_k  \leq \min \left\{ \frac{j}{2^b-1} + \alpha - z_{k+1}, \frac{j+1}{2^b-1} - z_{k+1}\right \}  \: \Big| \:  z\Bigr) \\
    & + \mathbb{P}\Bigl( \frac{j+1}{2^b-1} - z_{k+1}\leq u_k  \leq \min \left\{ \frac{j+2}{2^b-1} + \alpha - z_{k+1}, \frac{j+2}{2^b-1} - z_{k+1}\right \}  \: \Big| \:  z\Bigr)\\
    &  = \mathbb{P}\Bigl(  u_k  \leq \min \left\{ \frac{j}{2^b-1} + \alpha - z_{k+1}, \frac{j+1}{2^b-1} \right \}  \: \Big| \:  z\Bigr) \\
    & + \mathbb{P}\Bigl( \frac{j+1}{2^b-1} - z_{k+1} \leq u_k  \leq \min \left\{ \frac{j+2}{2^b-1} + \alpha - z_{k+1}, \frac{j+2}{2^b-1} - z_{k+1} \right \}  \: \Big| \:  z\Bigr).
\end{split}
\end{equation*}
Note that in the third equality we make use of the fact that $\frac{j}{2^b-1} \leq z_{k+1} \leq \frac{j+1}{2^b-1}$ implies $ \frac{j-1}{2^b-1} - z_{k+1} \leq -\frac{1}{2^b-1}$. Now note that
\begin{equation*}
\begin{split}
    & \mathbb{P}\Bigl(  u_k  \leq \min \left\{ \frac{j}{2^b-1} + \alpha - z_{k+1}, \frac{j+1}{2^b-1} \right \} \: \Big| \:  z\Bigr)\\
    & = 
    \begin{cases}
    0 & \text{ if }\ \alpha < z_{k+1} - \frac{j+1}{2^b-1},\\
    \int_{-\frac{1}{2^b-1}}^{\frac{j}{2^b-1} + \alpha - z_{k+1}} \frac{2^b-1}{2} = \frac{2^b-1}{2}\left(\frac{j+1}{2^b-1} + \alpha - z_{k+1} \right) & \text{ if }\ z_{k+1} - \frac{j+1}{2^b-1} \leq \alpha <   \frac{1}{2^b-1},\\
    \int_{-\frac{1}{2^b-1}}^{\frac{j+1}{2^b-1} - z_{k+1}} \frac{2^b-1}{2} = \frac{2^b-1}{2}\left(\frac{j+2}{2^b-1} - z_{k+1} \right)  & \text{ if } \ \alpha \geq  \frac{1}{2^b-1},
    \end{cases}
\end{split}
\end{equation*}
and
\begin{equation*}
\begin{split}
    & \mathbb{P}\Bigl( \frac{j+1}{2^b-1} -z_{k+1}\leq u_k  \leq \min \left\{ \frac{j+2}{2^b-1} + \alpha - z_{k+1}, \frac{j+2}{2^b-1} - z_{k+1}\right \}  \: \Big| \:  z\Bigr)\\
    & = 
    \begin{cases}
    0 & \text{ if }\  \alpha < -\frac{1}{2^b-1}, \\
    \int_{\frac{j+1}{2^b-1} - z_{k+1}}^{\frac{j+2}{2^b-1} + \alpha - z_{k+1}} \frac{2^b-1}{2} = \frac{2^b-1}{2}\left(\frac{1}{2^b-1} + \alpha \right) & \text{ if }\  -\frac{1}{2^b-1} \leq \alpha <   z_{k+1} - \frac{j+1}{2^b-1},\\
     \int_{\frac{j+1}{2^b-1} - z_{k+1}}^{\frac{1}{2^b-1}} \frac{2^b-1}{2} = \frac{2^b-1}{2}\left(z_{k+1} - \frac{j}{2^b-1} \right)  & \text{ if }\ \alpha \geq  z_{k+1} - \frac{j+1}{2^b-1}.
    \end{cases}
\end{split}
\end{equation*}
Thus
\begin{equation*}
    \mathbb{P}(u_{k+1} \leq \alpha \: | \: z)= 
    \begin{cases}
    0 & \text{if} \ \alpha < -\frac{1}{2^b-1},\\
    \frac{2^b-1}{2}\left(\frac{1}{2^b-1} + \alpha \right) &\text{if}\ -\frac{1}{2^b-1} \leq \alpha <   z_{k+1} - \frac{j+1}{2^b-1},\\
    \frac{2^b-1}{2}\left(\frac{1}{2^b-1} + \alpha \right) &\text{if}\ z_{k+1} - \frac{j+1}{2^b-1} \leq \alpha <   \frac{1}{2^b-1},\\
    0  & \text{if}\ \alpha \geq \frac{1}{2^b-1},
    \end{cases}
\end{equation*}
which shows that $u_{k+1} \: | \: z \sim  U\left[-\frac{1}{2^b-1},\frac{1}{2^b-1}\right]$.

Showing that $u_{k+1} \: | \: z \sim  U\left[-\frac{1}{2^b-1},\frac{1}{2^b-1}\right]$ for the other cases, namely, $\frac{j}{2^b-1} \leq z_{k+1} \leq \frac{j+1}{2^b-1}$ where $j \in \{ 0, 2,\cdots  2^b-2\}$
and  $-\frac{j+1}{2^b-1} \leq z_{k+1} \leq -\frac{j}{2^b-1}$ where $j \in \{ 1, 3,\cdots  2^b-3\}$ and  $-\frac{j+1}{2^b-1} \leq z_{k+1} \leq -\frac{j}{2^b-1}$ where $j \in \{ 0, 2,\cdots  2^b-2\}$ follow a similar argument as above and for the sake of brevity is skipped from an explicit mention. Thus, by induction, we have shown that $u_{k+1} \: | \: z \sim  U\left[-\frac{1}{2^b-1},\frac{1}{2^b-1}\right]$.
\end{proof}

For the subsequent sections, we adopt the following notations . Let $A$ be the matrix whose rows are the vectors $\{ (\Tilde{V}z(x))^T\}_x$, $B$ be the matrix whose rows are the vectors $\{ (\Tilde{V}D^ru_x)^T\}_x$ and $C$ be the matrix whose first column is $\frac{\sqrt{2}}{\sqrt{p}\|v\|_2}    (u_0^{x_1} \:  u_0^{x_2} \: \cdots \:  u_0^{x_n})^T$ and all other columns as zero. Let the columns of $A, B, C$ be denoted by $A_i, B_i, C_i$ respectively. Now the corresponding approximation to the kernel can written as
\[
             \Hat{K}_{\Sigma\Delta} = (A-B+C)(A-B+C)^T =  \sum_{i = 1}^p (A_i - B_i + C_i)(A_i - B_i + C_i)^T.
\]

\begin{lemma}
\label{uppbnds1}
$$\mathbb{E}[\Hat{K}_{\Sigma\Delta}] = K + \sum_{i=1}^p\Lambda_i \preccurlyeq K + \frac{1}{\lambda(2^b-1)^2}\left( 8  +   \frac{26}{3p} \right) I$$
where each $\Lambda_i$ is a diagonal matrix with positive diagonal entries, $\delta > 0$ and $I$ is the identity matrix.
\end{lemma}
\begin{proof}
We begin by noting that
\[
    \begin{split}
        \mathbb{E}[\Hat{K}_{\Sigma\Delta}] & = \mathbb{E}[\sum_{i = 1}^p (A_i - B_i + C_i)(A_i - B_i + C_i)^T]\\
        & = \mathbb{E}[\sum_{i = 1}^p A_iA_i^T] -  \mathbb{E}[\sum_{i = 1}^p A_iB_i^T] - 
        \mathbb{E}[\sum_{i = 1}^p B_iA_i^T] +
        \mathbb{E}[\sum_{i = 1}^p B_iB_i^T] + \mathbb{E}[\sum_{i = 1}^p A_iC_i^T]\\
        &
        + \mathbb{E}[\sum_{i = 1}^p C_iA_i^T]
        - \mathbb{E}[\sum_{i = 1}^p B_iC_i^T] -  \mathbb{E}[\sum_{i = 1}^p C_iB_i^T] +  \mathbb{E}[\sum_{i = 1}^p C_iC_i^T]\\
        & = K -\sum_{i = 1}^p\mathbb{E}[A_iB_i^T] - 
        \sum_{i = 1}^p\mathbb{E}[B_iA_i^T] +
        \sum_{i = 1}^p\mathbb{E}[B_iB_i^T]
        + \mathbb{E}[B_1C_1^T] +  \mathbb{E}[C_1B_1^T] +  \mathbb{E}[C_1C_1^T]
    \end{split}
\]
where we've used the result from lemma \ref{identities} that $\mathbb{E}[AA^T] = K$. Now, let $F := \Tilde{V}D^r \in \mathbb{R}^{p \times m}$ and $\{f_i^T\}_{i=1}^p$ denote the rows of $F$. Then let
$$
B_i := \begin{pmatrix}
f_i^Tu_{x_1}\\
f_i^Tu_{x_2}\\
\vdots\\
f_i^Tu_{x_n}\\
\end{pmatrix}
$$ 
where $B_i$ is the $i$-th column of $B$, $\{x_1, \cdots x_n\}$ represent the individual data samples and $u_{x_j} \in \mathbb{R}^m$ for each $j = 1, \cdots, n$. Note that $u_{x_j}$ (when conditioned on $Z$) across data points $x_j$ are independent with respect to each other with their entries $\sim U[-1/(2^b-1),1/(2^b-1)]$. Thus,
\[
    \begin{split}
        \mathbb{E}[A_iB_i^T] & =  \mathbb{E}[A_i \begin{pmatrix}
f_i^Tu_{x_1}&
f_i^Tu_{x_2}&
\cdots&
f_i^Tu_{x_n}
\end{pmatrix}]\\
        & = \mathbb{E}_Z[A_i\begin{pmatrix}
\mathbb{E}_{u_{x_1}}[f_i^Tu_{x_1} \: | \: z_{x_1}]&
\mathbb{E}_{u_{x_2}}[f_i^Tu_{x_2} \: | \: z_{x_2}]&
\cdots&
\mathbb{E}_{u_{x_n}}[f_i^Tu_{x_n} \: | \: z_{x_n}]
\end{pmatrix}]\\
& = 0.
    \end{split}
\]
By a similar argument, $\mathbb{E}[B_iA_i^T]  = 0$ and $\mathbb{E}[\sum_{i = 1}^p A_iC_i^T] = \mathbb{E}[\sum_{i = 1}^p C_iA_i^T] = 0$. Now,
\begin{equation*}
        \begin{split}
        \sum_{i = 1}^p\mathbb{E}[B_iB_i^T] & = \sum_{i = 1}^p \mathbb{E}_Z\mathbb{E}_U   \left[
        \begin{pmatrix}
            (f_i^Tu_{x_1}) (f_i^Tu_{x_1}) &  (f_i^Tu_{x_1}) (f_i^Tu_{x_2}) & \cdots  & (f_i^Tu_{x_1}) (f_i^Tu_{x_n})\\
            (f_i^Tu_{x_2}) (f_i^Tu_{x_1}) & (f_i^Tu_{x_2}) (f_i^Tu_{x_2}) & \cdots &  (f_i^Tu_{x_2}) (f_i^Tu_{x_n})\\
            \vdots & \vdots & \vdots & \vdots\\
             (f_i^Tu_{x_n}) (f_i^Tu_{x_1}) & \cdots & \cdots &  (f_i^Tu_{x_n}) (f_i^Tu_{x_n})
        \end{pmatrix} \right].
    \end{split}
\end{equation*}
First we note that $u_{x_j}$ (when conditioned on $Z$) across data points $x_j$ are independent with respect to each other with their entries $\sim U[-1/(2^b-1),1/(2^b-1)]$ and thus $\mathbb{E}[BB^T]$ is a diagonal matrix. Then note that each row of $VD$ has atmost $2$ non-zero entries which are either $\{1,-1\}$. Thus, 
\[
    \begin{split}
        |f_i^Tu_{x_i}| & = |\langle f_i, u_{x_j} \rangle| \leq \frac{\sqrt{2}}{\sqrt{p}\|v\|_2}  \frac{2}{2^b-1} = \frac{2^{3/2}}{\sqrt{p}(2^b-1)\|v\|_2}.
    \end{split}
\]
Further, since $r=1$, $\|v\|_2^2 = \lambda$ which implies.
$|f_i^Tu_{x_i}|  \leq \frac{2^{3/2}}{\sqrt{p\lambda}(2^b-1)}$
Thus, the diagonal matrix $\mathbb{E}[B_iB_i^T]  \preccurlyeq \frac{8}{p\lambda (2^b-1)^2} I$ which in turn implies $\mathbb{E}[BB^T]  \preccurlyeq  \frac{8}{\lambda (2^b-1)^2} I$.\\\\
Now,
\[
    \begin{split}
        \mathbb{E}[B_1C_1^T] & =   \frac{\sqrt{2}}{\sqrt{p}\|v\|_2}\mathbb{E}_Z\mathbb{E}_U
         \left[
        \begin{pmatrix}
            (f_1^Tu_{x_1}) (u_0^{x_1}) &  (f_1^Tu_{x_1}) (u_0^{x_2}) & \cdots  & (f_1^Tu_{x_1}) (u_0^{x_n})\\
            (f_1^Tu_{x_2}) (u_0^{x_1}) & (f_1^Tu_{x_2}) (u_0^{x_2}) & \cdots &  (f_1^Tu_{x_2}) (u_0^{x_n})\\
            \vdots & \vdots & \vdots & \vdots\\
             (f_1^Tu_{x_n}) (u_0^{x_1}) & \cdots & \cdots &  (f_1^Tu_{x_n}) (u_0^{x_n})
        \end{pmatrix} \right].
      \\
     \end{split}
\]
Thus, by similar reasoning as in prior paragraphs, $\mathbb{E}[B_1C_1^T]$ is a diagonal matrix and also
\[
    |u_0^{x}(f_1^Tu_{x})| \leq \frac{1}{2^b-1}|f_1^Tu_{x}| \leq \frac{2^{3/2}}{\sqrt{p\lambda}(2^b-1)^2}
\]
and thus
\[
    \frac{\sqrt{2}}{\sqrt{p}\|v\|_2} |u_0^{x}(f_1^Tu_{x})| \leq \frac{4}{p\lambda(2^b-1)^2}.
\]
So $\mathbb{E}[-B_1C_1^T] \preccurlyeq  \frac{4}{p\lambda(2^b-1)^2} I$. Similarly, $\mathbb{E}[-C_1B_1^T] \preccurlyeq  \frac{4}{p\lambda(2^b-1)^2} I$. Now,
\[
    \begin{split}
        \mathbb{E}[C_1C_1^T] & =  \frac{2}{p\|v\|_2^2}\mathbb{E}_Z\mathbb{E}_U
         \left[
        \begin{pmatrix}
            (u_0^{x_1}) (u_0^{x_1}) &  (u_0^{x_1}) (u_0^{x_2}) & \cdots  & (u_0^{x_1}) (u_0^{x_n})\\
            (u_0^{x_2}) (u_0^{x_1}) & (u_0^{x_2}) (u_0^{x_2}) & \cdots &  (u_0^{x_2}) (u_0^{x_n})\\
            \vdots & \vdots & \vdots & \vdots\\
             (u_0^{x_n}) (u_0^{x_1}) & \cdots & \cdots &  (u_0^{x_n}) (u_0^{x_n})
        \end{pmatrix} \right]\\
        & = \frac{2}{3p\lambda(2^b-1)^2}
      \\
     \end{split}
\]
and thus $\mathbb{E}[C_1C_1^T] \preccurlyeq  \left(\frac{2}{3 r^{-2r}}\right) \frac{1}{p\lambda^{2r-1}}$. Thus, putting together the bounds for each of the terms, we get
\[
    \begin{split}
        \mathbb{E}[\Hat{K}_{\Sigma\Delta}] & = K +  \mathbb{E}[BB^T] -
          \mathbb{E}[B_1C_1^T] - \mathbb{E}[C_1B_1^T] +  \mathbb{E}[C_1C_1^T]   \\
        & = K + \Lambda
    \end{split}
\]
where $\Lambda : =  \mathbb{E}[BB^T]  -  \mathbb{E}[B_1C_1^T] -  \mathbb{E}[C_1B_1^T] +  \mathbb{E}[C_1C_1^T]$ is a diagonal matrix and
\[
    \begin{split}
      \Lambda \preccurlyeq \frac{1}{\lambda(2^b-1)^2}\left[ 8  +   \frac{8}{p} + \frac{2}{3p} \right] I.
    \end{split}
\]
Thus $\mathbb{E}[\Lambda]\preccurlyeq \delta I$ where
\[
    \delta := \frac{1}{\lambda(2^b-1)^2}\left[ 8  +   \frac{26}{3p}\right].
\]
\end{proof}

\begin{lemma}[\cite{zhang2019low}]
\label{connect1}
Let $\eta>0$, $K$ and $\Hat{K}$ be positive symmetric semi-definite matrices, then
\begin{equation*}
      (1-\Delta_1)(K + \eta I) \preccurlyeq (\Hat{K} + \eta I) \preccurlyeq (1+\Delta_2)(K + \eta I)  \iff -\Delta_1I \preccurlyeq M(\Hat{K}  - K)M \preccurlyeq \Delta_2I
\end{equation*}
where, $M : = (K + \eta I)^{-1/2}$.
\end{lemma}

\begin{proof}
The proof is obtained using the following sequence of equivalent statements. 
\begin{equation*}
    \begin{split}
        & (1-\Delta_1)(K + \eta I) \preccurlyeq (\Hat{K} + \eta I) \preccurlyeq (1+\Delta_2)(K + \eta I)\\
        \iff & (1-\Delta_1)I \preccurlyeq (K + \eta I)^{-1/2}(\Hat{K} + \eta I)(K + \eta I)^{-1/2} \preccurlyeq (1+\Delta_2)I\\
        \iff & -\Delta_1I \preccurlyeq M(\Hat{K} + \eta I)M - I \preccurlyeq \Delta_2I\\
        \iff & -\Delta_1I \preccurlyeq M(\Hat{K} + \eta I)M -  (K + \eta I)^{-1/2}(K + \eta I) (K + \eta I)^{-1/2} \preccurlyeq \Delta_2I\\
        \iff & -\Delta_1I_n \preccurlyeq M(\Hat{K} + \eta I - K - \eta I_n)M \preccurlyeq \Delta_2I_n\\
        \iff & -\Delta_1I \preccurlyeq M(\Hat{K}  - K)M \preccurlyeq \Delta_2I.
    \end{split}
\end{equation*}
Note that the assumptions made on $K,\Hat{K}$ and $\eta$ imply that $\Hat{K} + \eta I$ is invertible and also $(K + \eta I)^{-1/2}$ exists.
\end{proof}

\begin{lemma}[\cite{zhang2019low}]
\label{connect2}
Let $0 \preccurlyeq \Lambda \preccurlyeq \delta I$ where $\delta >0$. Also let $\eta>0$, $K$ and $\Hat{K}$ be positive symmetric semi-definite matrices and $M:=(K + \eta I)^{-1/2}$. Then
\begin{equation*}
        -\Delta_1I_n \preccurlyeq M(\Hat{K}  - (K+ \Lambda ))M \preccurlyeq (\Delta_2 - \frac{\delta}{\eta})I_n \implies -\Delta_1I_n \preccurlyeq M(\Hat{K}  - K)M \preccurlyeq \Delta_2I_n.
\end{equation*}
\end{lemma}

\begin{proof}
We begin by noting that $0 \preccurlyeq M\Lambda M$ since $M$ is invertible, $0 \preccurlyeq \Lambda$ and for all $x \neq 0$, $x^TM \Lambda Mx = (Mx)^T \Lambda (Mx)$. Thus
\begin{equation*}
        -\Delta_1I_n \preccurlyeq M(\Hat{K}  - (K+ \Lambda ))M  \implies -\Delta_1I_n \preccurlyeq M(\Hat{K}  - K)M.
\end{equation*}

Also, note that $\|M\|_2^2 = \|M^2\|_2 = \|(K+\eta I)^{-1}\|_2$ ($M$ and $M^2$ are symmetric) and $\|M \Lambda M\| \leq \| \Lambda \|\|(K+\eta I)^{-1}\|$. Also since $0 \preccurlyeq K$ (positive semi-definite kernel), we have that $\|(K+\eta I)^{-1}\|_2 \leq \frac{1}{\eta}$. Hence, we get
\begin{equation*}
        -\Delta_1I_n \preccurlyeq M(\Hat{K} - (K+ \Lambda ))M \preccurlyeq (\Delta_2 - \frac{\delta}{\eta})I_n \implies -\Delta_1I_n \preccurlyeq M(\Hat{K}  - K)M \preccurlyeq \Delta_2I_n.
\end{equation*}
\end{proof}

\begin{theorem}[Matrix-Bernstein inequality \cite{zhang2019low}]
\label{matrixber}
Consider a finite sequence $\{S_i\}$ of random Hermitian matrices of the same size and assume that 
\[
\mathbb{E}[S_i] = 0 \quad \text{ and } \quad \lambda_{max}(S_i) \leq l \quad \text{ for each index } i.
\]
Let $S = \sum_i S_i$ and $\mathbb{E}[S^2] \preccurlyeq W$, i.e. $W$ is a semi-definite upper bound for the second moment of $S$. Then, for $t \geq 0$,
$$\mathbb{P}[\lambda_{max}(S) \geq t] \leq 4\frac{\text{tr}(W)}{\|W\|}.\exp(\frac{-t^2/2}{\|W\|+lt/3}).
$$
\end{theorem}

Recall our notations where $A$ is the matrix whose rows are the vectors $\{ (\Tilde{V}z(x))^T\}_x$, $B$ is the matrix whose rows are the vectors $\{ (\Tilde{V}D^ru_x)^T\}_x$ and $C$ is the matrix whose first column is $\frac{\sqrt{2}}{\sqrt{p}\|v\|_2}    (u_0^{x_1},  u_0^{x_2}, \cdots, u_0^{x_n})^T$ and all other columns as zero. Also the columns of $A, B, C$ are denoted by $A_i, B_i, C_i$ respectively. Additionally let $K_i : = \mathbb{E}[A_iA_i^T]$, $M := (K + \eta I_n)^{-1/2}$ and
\begin{equation}
\label{Si_definition}
        S_i  := M(A_i - B_i - C_i)(A_i - B_i - C_i)^TM^T -M(K_i +  \Lambda_i)M^T.
\end{equation}
Thus note that by design $\mathbb{E}[S_i] = 0$. We now will show that the remaining assumptions required to apply Matrix-Bernstein inequality hold for the sequence of matrices $\{S_i\}_{i=1}^p$.

\begin{lemma}
\label{uppbnds2}
The $2$-norm of $S_i$ (defined in \eqref{Si_definition}) is bounded for each $i=1, \cdots, p$ and $\mathbb{E}[S]^2$ has a semi-definite upper bound, where $S = \sum_i S_i$. In particular,
$$\|S_i\| \leq  \frac{2n\lambda}{p\eta^2} \quad (:= l) \qquad \text{and} \qquad \mathbb{E}[S^2] \preccurlyeq l\Tilde{M}$$
where, $n$ is the number of data samples, $\eta$ is the regularization, $m=\lambda p$ denotes the parameters of $\Sigma\Delta$ quantization and $\Tilde{M} := M(K+\Lambda)M^T$.
\end{lemma}
\begin{proof}
(i) $\lambda_{max}(S_i)$ is bounded.
Let $u_i := M(A_i-B_i-C_i)$, then $S_i = u_iu_i^T - \mathbb{E}[u_iu_i^T]$. First note that
\[
    \begin{split}
        \|u_iu_i^T\| & = \|u_i\|^2\\
        & = \|  M(A_i-B_i-C_i)\|^2\\
        & \leq \|M\|^2\|A_i - B_i - C_i \|^2.
    \end{split}
\]
Also, $A_i-B_i-C_i$ is the $i$-th column of the matrix which has as its rows the vectors $\{\Tilde{V}Q(z(x))^T\}_x$. Thus, in general  
$$A_i-B_i-C_i = \begin{pmatrix}
g_i^Tq_{x_1}\\
g_i^Tq_{x_2}\\
\vdots\\
g_i^Tq_{x_n}\\
\end{pmatrix}
$$ 
where, $g_i^T$ denotes the $i$-th row of $\Tilde{V}$. Also note that the entries of $Q(z(x))$ are in $\mathcal{A} = \Bigl\{\frac{a}{2K-1}\, \Big|\, a= \pm 1,\pm 3, \ldots, \pm (2K-1)  \Bigr\}$. Thus, 
\[
    \begin{split}
        \|A_i - B_i - C_i\|_2^2  & \leq n\|g_i\|_1^2\\
        & \leq n\|\Tilde{V}\|_{\infty}^2.
    \end{split}
\]
Note that $\Tilde{V} = \frac{\sqrt{2}}{\sqrt{p}\| v\|_2}(I_p \otimes v)$ where for $r=1$, $v \in \mathbb{R}^{\lambda}$ is the vector of all ones, which implies $\|\Tilde{V}\|_{\infty} =  \frac{\sqrt{2\lambda}}{\sqrt{p}}$. Thus,
\[
    \|M\|^2\|A_i - B_i - C_i\|^2 \leq \frac{2n\lambda}{p}\|M\|^2.
\]
Further, since by definition $M = (K+\eta I)^{-1/2}$ we have,
\[
    \begin{split}
        \|M\|^2\|A_i - B_i - C_i\|^2 & \leq \frac{2n\lambda}{p}\|M\|^2 = \frac{2n\lambda}{p}\|(K+\eta I)^{-1}\|^2\\
        & \leq \frac{2n\lambda}{p\eta^2} \quad (:= l).
    \end{split}
\]
Thus we see that,
\[
    \begin{split}
        \|S_i \| & =  \|u_iu_i^T - \mathbb{E}[u_iu_i^T]\|\\
        & \leq \|u_iu_i^T\| + \|\mathbb{E}[u_iu_i^T]\|\\
        & \leq 2l.
    \end{split}
\]
So $\|S_i\| \leq 2l$ implies $\lambda_{max}(S_i) \leq 2l$. Note that $K_i$ and $ \Lambda_i$ are expectations of symmetric matrices and thus $S_i$ is symmetric. 

\noindent (ii) $\mathbb{E}[S^2]$ has a semi-definite upper bound.
\[
    \begin{split}
        \mathbb{E}[S_i^2] & = \mathbb{E}[(u_iu_i^T)^2] - \mathbb{E}[u_iu_i^T]^2\\ 
        & \preccurlyeq \mathbb{E}[(u_iu_i^T)^2]
         = \mathbb{E}[\|u_i\|^2u_iu_i^T]\\
        & \preccurlyeq l \: \mathbb{E}[u_iu_i^T].
    \end{split}
\]
Now,
\[
    \begin{split}
        \mathbb{E}[S^2] & = \sum_{i=1}^p\mathbb{E}[S_i^2]\\
         & \preccurlyeq l \sum_{i=1}^p \mathbb{E}[u_iu_i^T]
         = l \sum_{i=1}^p M(K_i +  \Lambda_i)M^T \\
        & \preccurlyeq lM(K+ \Lambda)M^T\\
        & \preccurlyeq l\Tilde{M}
    \end{split}
\]
where $\Tilde{M} := M(K+ \Lambda)M^T$ and thus $\mathbb{E}[S^2] \preccurlyeq l\Tilde{M}$.
\end{proof}

Now we are in a position to prove theorem \ref{specBoundtheorem} of the main text which we restate for convenience.

\begin{theorem}
Let $\Hat{K}_{\Sigma\Delta}$ be an approximation of a true kernel matrix $K$ using $m$-feature first-order $\Sigma\Delta$ quantized RFF (as in \eqref{sigdelrand2}) with a $b$-bit alphabet (as in \eqref{alphabet}) and $m=\lambda p$. Then given $\Delta_1 \geq 0, \Delta_2 \geq \frac{\delta}{\eta}$ where $\eta>0$ represents the regularization and $\delta = \frac{8  +   \frac{26}{3p}}{\lambda(2^b-1)^2}$, we have
\[
\begin{split}
    & \mathbb{P}[ (1-\Delta_1)(K + \eta I) \preccurlyeq (\Hat{K}_{\Sigma\Delta} + \eta I) \preccurlyeq (1+\Delta_2)(K + \eta I)] \\
    & \geq \quad 1 - 4n\left[\exp(\frac{-\Delta_1^2/2}{l(\frac{1}{\eta}(\|K\|_2 + \delta)+2\Delta_1/3)}) + \exp(\frac{-(\Delta_2 - \frac{\delta}{\eta})^2/2}{l(\frac{1}{\eta}(\|K\|_2 + \delta)+2(\Delta_2 - \frac{\delta}{\eta})/3)}) \right] 
\end{split}
\]
where, $l = \frac{2n\lambda}{p\eta^2}$.
\end{theorem}

\begin{proof}
We apply Matrix-Bernstein inequality (theorem \ref{matrixber}) to $\{S_i\}_{i=1}^p$ (defined in \ref{Si_definition}) to obtain that given $t_2 \geq 0$,
\begin{equation*}
        \mathbb{P}[\lambda_{max}(M(\Hat{K}_{\Sigma\Delta} - (K+ \Lambda))M^T) \geq t_2 ] \leq 4\frac{\text{tr}(\Tilde{M})}{\|\Tilde{M}\|}\exp(\frac{-t_2^2/2}{l\|\Tilde{M}\|+2lt_2/3}).
\end{equation*}
Now, since $\lambda_{max}(S) = -\lambda_{min}(-S)$, by repeating an identical argument for $-S$ we obtain that given $t_1 \geq 0$,
\begin{equation*}
        \mathbb{P}[ \lambda_{min}(M(\Hat{K}_{\Sigma\Delta} - (K+ \Lambda))M^T)  \leq -t_1] \leq 4\frac{\text{tr}(\Tilde{M})}{\|\Tilde{M}\|}\exp(\frac{-t_1^2/2}{l\|\Tilde{M}\|+2lt_1/3}).
\end{equation*}
Putting the above two equations together with the fact that $M =  (K + \eta I_n)^{-1/2}$  we obtain that for $t_1, t_2 \geq 0$,
\begin{equation*}
    \begin{split}
            & \mathbb{P}[ -t_1I_n \preccurlyeq M(\Hat{K}_{\Sigma\Delta}  - (K+ \Lambda))M \preccurlyeq t_2I_n] \\
            & \geq 1 - 4\frac{\text{tr}(\Tilde{M})}{\|\Tilde{M}\|}\left[\exp(\frac{-t_1^2/2}{l\|\Tilde{M}\|+2lt_1/3}) + \exp(\frac{-t_2^2/2}{l\|\Tilde{M}\|+2lt_2/3}) \right].
    \end{split}
\end{equation*}
Thus, by lemmas \ref{uppbnds1}, \ref{connect1}, \ref{connect2} and \ref{uppbnds2}, for the $\Sigma\Delta$-quantized RFF kernel $\Hat{K}_{\Sigma\Delta}$, given $\Delta_1 \geq 0, \Delta_2 \geq \frac{\delta}{\eta}$ we have the following spectral approximation result:
\[
\begin{split}
    \mathbb{P}[ (1-\Delta_1)(K + \eta I) \preccurlyeq (\Hat{K}_{\Sigma\Delta} + \eta I) \preccurlyeq (1+\Delta_2)(K + \eta I)] \\
    \geq \quad 1 - 4\frac{\text{tr}(\Tilde{M})}{\|\Tilde{M}\|}\left[\exp(\frac{-\Delta_1^2/2}{l(\|\Tilde{M}\|+2\Delta_1/3)}) + \exp(\frac{-(\Delta_2 - \frac{\delta}{\eta})^2/2}{l(\|\Tilde{M}\|+2(\Delta_2 - \frac{\delta}{\eta})/3)}) \right] 
\end{split}
\]
where  $\Tilde{M} = M(K+ \Lambda)M$, $M = (K + \eta I_n)^{-1/2}$, $l=\frac{2n\lambda}{p\eta^2}$ is the upper bound for $\|u_iu_i^T\|$ computed in lemma \ref{uppbnds2} and $\delta = \frac{8  +   \frac{26}{3p}}{\lambda(2^b-1)^2}$ is the bound computed in lemma \ref{uppbnds1} such that $\mathbb{E}[\hat{K}_{\Sigma\Delta}] \preccurlyeq K + \delta I$.
Now, note that 
\begin{equation*}
        \begin{split}
        \| \Tilde{M}\|_2 & = \|M(K+\Lambda)M \|_2\\
        & \leq \|M\|_2^2\|K+\Lambda\|_2\\
        & \leq \|M\|_2^2(\|K\|_2 + \|\Lambda \|_2)\\
        & = \|(K+\eta I)^{-1}\|_2(\|K\|_2 + \|\Lambda \|_2)\\
        & \leq \frac{1}{\eta}(\|K\|_2 + \delta).
    \end{split}
\end{equation*}
Also given the positive semi-definite matrix $\Tilde{M}$, we know that $\|\Tilde{M}\|_2= \lambda_{max}(\Tilde{M})$ and thus $tr(\Tilde{M}) \leq \text{rank}(\Tilde{M})\|M\|_2$ which implies
$\frac{\text{tr}(\Tilde{M})}{\|\Tilde{M}\|}  \leq \text{rank}(\Tilde{M}) \leq n$. Hence,
\begin{equation*}
    \begin{split}
    & \mathbb{P}[ (1-\Delta_1)(K + \eta I) \preccurlyeq (\Hat{K}_{\Sigma\Delta} + \eta I) \preccurlyeq (1+\Delta_2)(K + \eta I)] \\
    & \geq \quad 1 - 4n\left[\exp(\frac{-\Delta_1^2/2}{l(\frac{1}{\eta}(\|K\|_2 + \delta)+2\Delta_1/3)}) + \exp(\frac{-(\Delta_2 - \frac{\delta}{\eta})^2/2}{l(\frac{1}{\eta}(\|K\|_2 + \delta)+2(\Delta_2 - \frac{\delta}{\eta})/3)}) \right].
\end{split}
\end{equation*}
\end{proof}

\end{document}